\documentclass{article}

\usepackage{microtype}
\usepackage{graphicx}
\usepackage{subfigure}
\usepackage{booktabs} 

\usepackage{hyperref}

\usepackage[accepted]{icml2023}

\usepackage{amsmath}
\usepackage{amssymb}
\usepackage{mathtools}
\usepackage{amsthm}

\usepackage[capitalize,noabbrev]{cleveref}

\theoremstyle{plain}
\newtheorem{theorem}{Theorem}[section]

\newtheorem{lemma}[theorem]{Lemma}
\newtheorem{corollary}[theorem]{Corollary}
\theoremstyle{definition}
\newtheorem{definition}[theorem]{Definition}

\theoremstyle{remark}
\newtheorem{remark}[theorem]{Remark}

\usepackage[textsize=tiny]{todonotes}

\usepackage{multirow}
\usepackage{enumitem}

\newcounter{mainsection}  

\theoremstyle{plain}
\newtheorem{maintheorem}{Theorem}[mainsection]
\newtheorem{maincorollary}[maintheorem]{Corollary}

\theoremstyle{definition}
\newtheorem{maindefinition}[maintheorem]{Definition}

\theoremstyle{remark}

\def\ddefloop#1{\ifx\ddefloop#1\else\ddef{#1}\expandafter\ddefloop\fi}
\def\ddef#1{\expandafter\def\csname bb#1\endcsname{\ensuremath{\mathbb{#1}}}}
\ddefloop ABCDEFGHIJKLMNOPQRSTUVWXYZ\ddefloop

\def\ddef#1{\expandafter\def\csname c#1\endcsname{\ensuremath{\mathcal{#1}}}}
\ddefloop ABCDEFGHIJKLMNOPQRSTUVWXYZ\ddefloop

\def\ddef#1{\expandafter\def\csname v#1\endcsname{\ensuremath{\boldsymbol{#1}}}}
\ddefloop ABCDEFGHIJKLMNOPQRSTUVWXYZabcdefghijklmnopqrstuvwxyz\ddefloop

\def\ddef#1{\expandafter\def\csname v#1\endcsname{\ensuremath{\boldsymbol{\csname #1\endcsname}}}}
\ddefloop {alpha}{beta}{gamma}{delta}{epsilon}{varepsilon}{zeta}{eta}{theta}{vartheta}{iota}{kappa}{lambda}{mu}{nu}{xi}{pi}{varpi}{rho}{varrho}{sigma}{varsigma}{tau}{upsilon}{phi}{varphi}{chi}{psi}{omega}{Gamma}{Delta}{Theta}{Lambda}{Xi}{Pi}{Sigma}{varSigma}{Upsilon}{Phi}{Psi}{Omega}{ell}\ddefloop

\DeclareMathOperator*{\argmin}{arg\,min}

\allowdisplaybreaks

\newcommand{\meansd}[2]{$#1 \pm #2$}

\def\algname{\textsc{FedCollab}}
\newcommand{\add}[1]{#1}
\renewcommand{\paragraph}[1]{\textbf{#1}\ \ }

\icmltitlerunning{Optimizing the Collaboration Structure in Cross-Silo Federated Learning}

\begin{document}

\twocolumn[
\icmltitle{Optimizing the Collaboration Structure in Cross-Silo Federated Learning}



\icmlsetsymbol{equal}{*}

\begin{icmlauthorlist}
\icmlauthor{Wenxuan Bao}{uiuc}
\icmlauthor{Haohan Wang}{uiuc}
\icmlauthor{Jun Wu}{uiuc}
\icmlauthor{Jingrui He}{uiuc}
\end{icmlauthorlist}

\icmlaffiliation{uiuc}{University of Illinois Urbana-Champaign, Champaign, IL, USA}

\icmlcorrespondingauthor{Jingrui He}{jingrui@illinois.edu}

\icmlkeywords{Federated Learning, Personalized Federated Learning, Negative Transfer}

\vskip 0.3in
]



\printAffiliationsAndNotice{}  

\begin{abstract}
    In federated learning (FL), multiple clients collaborate to train machine learning models together while keeping their data decentralized. Through utilizing more training data, FL suffers from the potential negative transfer problem: the global FL model may even perform worse than the models trained with local data only. In this paper, we propose {\algname}, a novel FL framework that alleviates negative transfer by clustering clients into non-overlapping coalitions based on their distribution distances and data quantities. As a result, each client only collaborates with the clients having similar data distributions, and tends to collaborate with more clients when it has less data. We evaluate our framework with a variety of datasets, models, and types of non-IIDness. Our results demonstrate that {\algname} effectively mitigates negative transfer across a wide range of FL algorithms and consistently outperforms other clustered FL algorithms.
\end{abstract}

\section{Introduction}

Federated learning (FL) is a distributed learning system where multiple clients collaborate to train a machine learning model under the orchestration of the central server, while keeping their data decentralized~\cite{fedavg}. We focus on \textit{cross-silo} FL, where clients are organizations with data that differ in their \textit{distributions} and \textit{quantities} \cite{leaf}. For example, the clients can be hospitals with varying patient types and numbers (e.g., children's hospitals, trauma centers). Although cross-silo FL clients can train \textit{local models} with their own data locally (\textit{local training}), they participate in FL for a model trained with more data, which potentially performs better than local models. 

Traditionally, global FL (GFL)~\cite{fedavg,fednova,fedprox} trains a single \textit{global model} for all clients that minimizes a weighted average of local losses. It is a natural solution when clients have independent and identically distributed (IID) data. However, when clients have non-IID data, GFL may suffer from the \textit{negative transfer} problem: the global model performs even worse than the local models \cite{fedfomo}. The negative transfer problem also plagues many personalized FL (PFL) algorithms \cite{per-fedavg,pfedme,ditto}. Although these algorithms allow each client to train a personalized model with parameters different from the global model, the regularization of the global model still prevents personalized models from achieving better performance than local models. 

One way to avoid negative transfer is clustered FL (CFL) \cite{cfl,fesem,kmeans}. CFL groups clients with similar data distribution into coalitions, and trains FL models within each coalition. As a result, each client only collaborates with other clients in the same coalition. By changing the collaboration structure, CFL can alleviate negative transfer with almost no additional computation and communication costs. 

\begin{figure}
\vspace{2mm}
    \begin{center}
        \includegraphics[width=0.9\columnwidth]{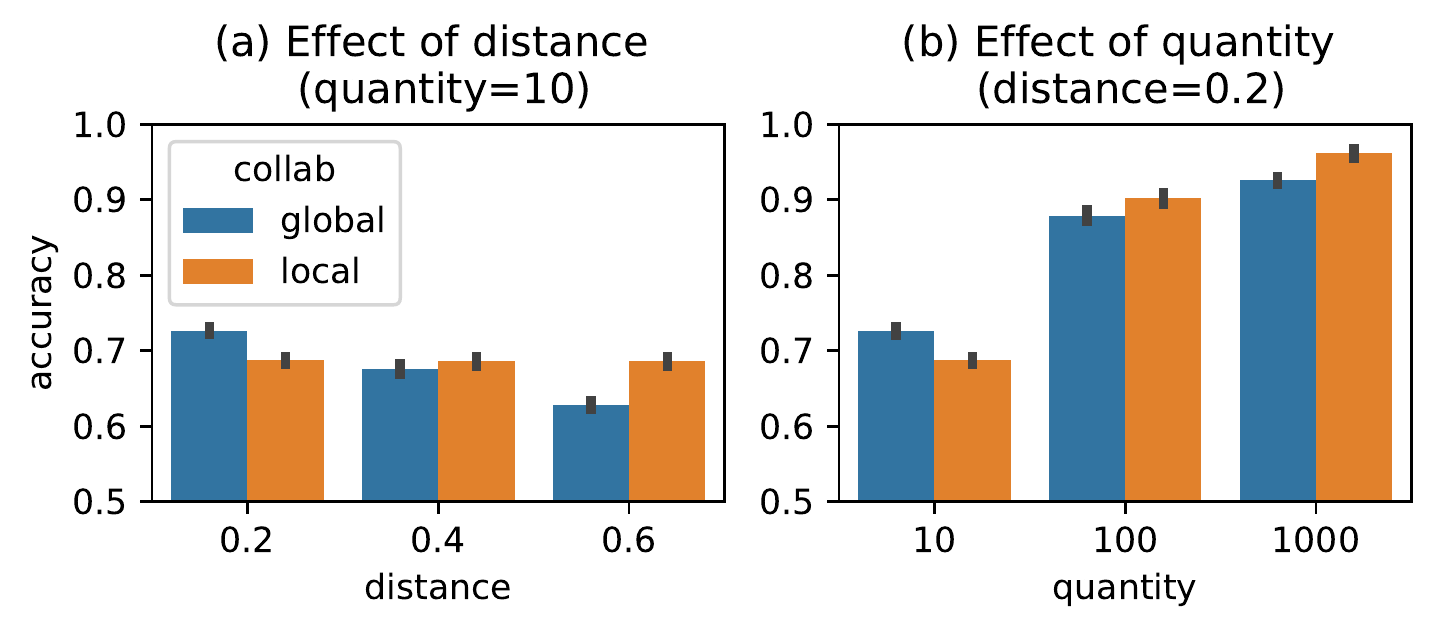}
        \caption{Effects of distribution distance and data quantity. When two clients have large distribution distance or large quantities, the local model performs better than global model. \label{fig:intro}}
    \end{center}
    \vspace{-5mm}
\end{figure}

A natural follow-up question would be what determines the best collaboration structure. We summarize two key factors for characterizing client collaboration: distribution distance and data quantity. We start with a simple 2-client scenario, studying whether the global model or the local model has higher accuracy. As shown in Figure \ref{fig:intro}, when two clients have small distribution distance and small quantity, the global model has higher accuracy, which is a scenario suitable for GFLs. When the distribution distance between two clients increases (Figure \ref{fig:intro}(a)), the local model performs better than the global model, showing that distribution distance influences the optimal collaboration structure. Meanwhile, it is often ignored that the data quantity also influences the optimal collaboration structure: given the same distribution distance, when the data quantity increases (Figure \ref{fig:intro}(b)), the local model also performs better than the global model. In other words, clients with more data are more ``picky'' in the choice of collaborators. 

Previous CFL algorithms \cite{ifca,fesem,cfl} generate clusters mainly based on loss values or parameter/gradient similarities, which have the following drawbacks. First, they ignore the influence of quantities, and group clients together, even when these clients have large quantities and prefer local training. Second, most CFL algorithms rely on indirect information of distribution distance, which does not recover the real distribution distance given the high complexity of neural networks (e.g., nonlinearity and permutation invariance). Finally, most CFL algorithms optimize the model parameters and collaboration structure simultaneously. Thus, they reinforce the current collaboration structure and fall into local optima easily, resulting in sub-optimal model performance. 

In this paper, we propose {\algname} to optimize for a better collaboration structure. First, we derive a theoretical error bound for each client in the FL system. The error bound consists of three terms: an irreducible minimal error term related to the model and data noise, a generalization error term depending on data quantities, and a dataset shift term depending on pairwise distribution distance between clients. By minimizing the error bounds, {\algname} solves for the optimal collaboration structure with awareness of both quantities and distribution distances. Second, to better estimate pairwise distribution distances without violating the privacy constraint of FL, we use a light-weight client discriminator between each pair of clients to predict which client the labeled data comes from, and train the discriminator within the FL framework. Third, we design an efficient optimization method to minimize the error bound. It requires no model training and solves the collaboration in seconds. Finally, we run FL algorithms within each coalition we identify. Since the model training and collaboration structure optimization are disentangled, {\algname} can be seamlessly integrated with any GFL or PFL algorithms. 

\paragraph{Contributions}
We summarize our contributions below.
\begin{itemize}[itemsep=0pt,topsep=0pt]
    \item We derive error bounds for FL clients and summarize two key factors that affect the model performance for each client: data quantity and distribution distance. (Section \ref{sec:analysis})
    \item We propose {\algname} to solve for the best collaboration structure, including a distribution difference estimator and an efficient optimizer. (Section \ref{sec:method})
    \item We empirically test our algorithm with a wide range of datasets, models, and types of non-IIDness. {\algname} enhances a variety of FL algorithms by providing better collaboration structures, and outperforms existing CFL algorithms in accuracy. (Section \ref{sec:experiments}) 
\end{itemize}

\section{Related Works}\label{sec:related_works}

\paragraph{Global Federated Learning}
Global federated learning (GFL) aims to train a single global model for private clients, by assuming that all the clients follow the same data distribution. Typically, FedAvg \cite{fedavg} is proposed to minimize a weighted average of local client objectives (e.g., empirical risks). More recently, many efforts \cite{fedprox,scaffold} have been made to speed up the convergence of FL on top of FedAvg. Another related line of works \cite{afl,qffl} is performance fairness aware federated learning, which encourages a uniform distribution of accuracy among clients. 
However, it is revealed \cite{fedfomo} that under severe data heterogeneity among clients, these GFL algorithms suffer from \textit{negative transfer} with undesirable performance on local clients.

\paragraph{Personalized Federated Learning}
In recent years, personalized federated learning has been proposed to deal with statistical data heterogeneity among clients. We roughly group them into two categories: coarse-grained and fine-grained. For coarse-grained PFL~\cite{per-fedavg,pfedme,ditto}, each client can further optimize a global model (trained with the union of local datasets) with its own data. This kind of PFL algorithm cannot choose which clients to collaborate with, and suffer from negative transfer when the client's own data distribution is distinct from the population. For fine-grained PFL~\cite{mocha}, clients can directly collaborate with some of the other clients. However, most of the fine-grained PFL algorithms significantly change the communication protocol of FL or introduce additional communication and computation costs \cite{mocha,fedfomo}. 

\paragraph{Clustered Federated Learning}
Similar to our algorithm, clustered federated learning partitions clients into clusters. For example, IFCA \cite{ifca} initializes multiple models and lets each client choose one based on the training loss; FeSEM \cite{fesem} lets each client choose a cluster with similar weights; and CFL \cite{cfl} iteratively bipartition the clients based on their cosine similarity of gradients. However, all these methods only consider distribution distances and ignore the importance of data quantities, which also play a key role in collaboration performance. To the best of our knowledge, \cite{modelshare} is the only work that considers the quantity in the optimization of the collaboration structure. However, it is limited to linear models with analytical solutions, and only considers a simplified non-IID setting.

\section{Analysis of Client Error Bound} \label{sec:analysis}

In this section, we derive a theoretical error bound to understand how data quantity and distribution distance affect the model performance for each client. 

\subsection{Setup} \label{subsec:analysis_setup}

We consider a FL system with $N$ clients connected to a central server. Each client $i \in \{1, \cdots, N\}$ has a dataset $\hat\cD_i = \{(\vx_k^{(i)}, \vy_k^{(i)})\}_{k=1}^{m_i}$ with $m_i$ samples drawn from its underlying true data distribution $\cD_i$, where $\vx_k^{(i)} \in \cX$ is the feature and $\vy_k^{(i)} \in \cY$ is the label. We denote $m = \sum_{i =1}^N m_i$ as the \textit{total quantity} of samples and $\vbeta = [\beta_1, \cdots, \beta_N] = [\frac{m_1}{m}, \cdots, \frac{m_N}{m}]$ as the \textit{client quantity distribution}. Given a machine learning model (hypothesis) $h$ and risk function $\ell$, client $i$'s \textit{local expected risk} is given by $\epsilon_i(h) = \bbE_{(\vx, \vy) \in \cD_i} \ell(h(\vx), \vy)$, and its \textit{local empirical risk} is given by $\hat{\epsilon}_i(h) = \frac{1}{m_i}\sum_{k=1}^{m_i} \ell(h(\vx_k^{(i)}), \vy_k^{(i)})$. The goal of each client $i \in \{1, \cdots, N\}$ is to find a model $h$ within the hypothesis space $\cH$ that minimizes its local expected risk, which we denote as $h_i^* = \argmin_{h \in \cH} \epsilon_i(h)$. However, clients can only optimize their models with their finite samples $\hat\cD_1, \cdots, \hat\cD_N$. There are several representative options: local training, global FL (GFL), and clustered FL (CFL). 

\paragraph{Local Training}
In local training, each client trains its own model individually without sharing information with other clients. Each local model minimizes the local empirical risk $\hat{h}_i = \argmin_{h \in \cH} \hat{\epsilon}_i(h)$. Despite its simplicity, local training can only utilize each client's local data, which impedes the generalization performance of local models. 

\paragraph{GFL}
FL provides a way for each client to utilize other clients' data to enhance the model, without directly exchanging raw data. In typical global FL algorithms \cite{fedavg,fedprox,fednova}, clients globally train a model to minimize an average of local empirical risks weighted by each client's data quantity, i.e., $\hat{h}_{\vbeta} = \argmin_{h \in \cH} \sum_{i=1}^N \beta_i \hat{\epsilon}_i(h)$. When all clients have the same underlying distribution $\cD_1 = \cdots = \cD_N$, GFL enlarges the ``training set'' with IID data, which improves the model generalization performance from a theoretical perspective. However, when clients have different distributions, the global model significantly degrades, and even performs worse than local training \cite{fedfomo}, which we refer to as the negative transfer problem. 

\paragraph{CFL}
More generally, the CFL framework partitions clients into non-overlapping coalitions and allows each client to train models only with clients in the same coalition. Clients in the same coalition share the same model, while clients in different coalitions have different model parameters. For a client $i$ in coalition $\cC$, it trains a model with all other clients in $\cC$ to minimize a weighted average of local empirical risks weighted by $\valpha_i = [\alpha_{i1}, \cdots \alpha_{iN}]$: 
\begin{align}
    \hat{h}_{\valpha_i} = \argmin_{h \in \cH} \sum_{j=1}^N \alpha_{ij} \hat{\epsilon}_j(h) \label{eq:cfl_loss}
\end{align}
where $\alpha_{ij} = \frac{\beta_j \cdot \bbI\{j \in \cC \}}{\sum_{k \in \cC} \beta_k} $ ($\bbI$ is the indicator function). By finding a good collaboration structure, CFL groups clients with similar distributions into the same coalition, so they can enjoy better generalization without suffering a lot of negative transfer. Notice that the objective (\ref{eq:cfl_loss}) subsumes both local training and GFL, by setting $\valpha_i$ as a one-hot vector (i.e., $\alpha_{ii} = 1$ and $\alpha_{ij} = 0$ for $j \neq i$) and $\valpha_i = \vbeta$. 

Given various collaboration options above, a natural question rises: \textit{which collaboration structure is optimal for client $i$}, i.e., having the lowest local expected risk $\epsilon_i(h)$? Since there are at least $2^{N-1}$ different coalitions for client $i$, it is prohibitively expensive to enumerate every option and pick the best model. Instead, in the next part, we derive a theoretical error bound for each client to estimate the error without training machine learning models practically. 

\subsection{Theoretical Error Bound}
Before deriving the generalization error bound for FL, we first introduce two concepts: quantity-aware function and distribution difference. 

\begin{definition}[Quantity-aware function] \label{def:gen}
    For a given hypothesis space $\cH$, combination weights $\valpha_i$, quantity distribution $\vbeta$, total quantity $m$, for any $\delta \in (0, 1)$, with probability at least $1 - \delta$ (over the choice of the samples), a quantity-aware function $\phi_{|\cH|}(\valpha_i, \vbeta, m, \delta)$ satisfies that for all $h \in \cH$, 
    \begin{align}
        | \hat\epsilon_{\valpha_i}(h) - \epsilon_{\valpha_i}(h) | \leq \phi_{|\cH|}(\valpha_i, \vbeta, m, \delta)
    \end{align}
\end{definition}

The quantity-aware function can be quantified with traditional generalization error bounds, including VC dimension \cite{ben2010theory} and weighted Rademacher complexity \cite{weighted_rad} (see Appendix \ref{sec:proof_of_theorem_3_1}). For example, when using VC dimension $d$ to quantify the complexity of hypothesis space $\cH$, we have
\begin{align}
\begin{split}
    & \phi_{|\cH|}(\valpha_i, \vbeta, m, \delta) \\
    & = \sqrt{\left( \sum_{j=1}^N \frac{\alpha_{ij}^2}{\beta_j} \right) \left( \frac{2d \log (2m + 2) + \log(4 / \delta)}{m} \right)}
\end{split}
\label{def:vc}
\end{align}

\begin{definition}[Distribution difference] \label{def:dif}
    For a given hypothesis space $\cH$, the distribution difference satisfies that for any two distributions $\cD_i, \cD_j$, the following holds for all $h \in \cH$, 
    \begin{align}
        | \epsilon_i(h) - \epsilon_j(h) | \leq D(\cD_i, \cD_j)
    \end{align}
\end{definition}

Distribution difference can also be quantified with a variety of distribution distances, including $\cH \Delta \cH$-distance \cite{ben2010theory} and $\cC$-divergence \cite{label_disc,c_divergence}. When using $\cC$-divergence, we have
\begin{align}
    D(\cD_i, \cD_j) = \max_{h \in \cH} \left| \epsilon_i(h) - \epsilon_j(h) \right|
\end{align}

\begin{theorem}\label{thm:error}
Let $\hat{h}_{\valpha_i}$ be the empirical risk minimizer defined in Eq. (\ref{eq:cfl_loss}) and $h_i^*$ be client $i$'s expected risk minimizer. For any $\delta \in (0, \frac{1}{2})$, with probability at least $1 - 2 \delta$, the following holds
\begin{align}
\begin{split}
    \epsilon_i(\hat{h}_{\valpha_i}) \leq \epsilon_i( h_i^* ) & + 2 \phi_{|\cH|}(\valpha_i, \vbeta, m, \delta) \\ 
    & + 2 \sum_{j \neq i} \alpha_{ij} D(\cD_i, \cD_j)
\end{split}
\end{align}
where $\epsilon_i( h_i^* ) = \min_{h \in \cH}\epsilon_i( h )$ is the minimal local expected risk that cannot be optimized given the distribution $\cD_i$ and the hypothesis space $\cH$.
\end{theorem}

Theorem \ref{thm:error} reveals that when we form a coalition for client $i$ to minimize its local expected risk $\epsilon_i(\hat{h}_{\valpha_i})$, both quantity information $(\vbeta, m)$ and distribution difference $\{D(\cD_i, \cD_j)\}_{i, j}$ should be considered. To better understand how Theorem \ref{thm:error} can guide the clustering of clients, we consider two special cases in Corollary \ref{crl:special} below. 

\begin{corollary} \label{crl:special}
When using VC-dimension bound (\ref{def:vc}) as the quantity aware function, the following results hold. 
\begin{itemize}[itemsep=0pt,topsep=0pt]
    \item If $D(\cD_i, \cD_j) = 0, \forall i, j$, GFL minimizes the error bound of Theorem \ref{thm:error} with $\alpha_{ij} = \beta_j, \forall {j}$. 
    \item If $\min_{j \neq i} D(\cD_i, \cD_j) > \frac{\sqrt{2d \log (2m + 2) + \log (4  / \delta)}\sqrt{m}}{2 m_i}$, local training minimizes the error bound of Theorem \ref{thm:error} with $\alpha_{ii} = 1$ and $\alpha_{ij} = 0, \forall j \neq i$. 
\end{itemize}
\end{corollary}

Corollary \ref{crl:special} matches with our observation in FL. GFL is most powerful when clients have the same data distribution. However, with large distribution distance and data quantity, local training becomes a better option. More generally, Corollary \ref{crl:general} shows that \textit{clients with more data are more ``picky'' in the choice of collaborators}. When a client $i$ has $m_i$ samples, it will only choose collaborators from clients with distribution difference smaller than or equal to $D_\text{thr}$, which decreases with the increase of $m_i$. 

\begin{corollary} \label{crl:general}
When using VC-dimension bound (\ref{def:vc}) as the quantity aware function, for a client $i$ with $m_i$ samples, if its coalition $\cC$ minimizes the error bound of Theorem \ref{thm:error}, then $\cC$ does not include any clients with distribution distance $D(\cD_i, \cD_j) > D_\text{thr}$, where $D_\text{thr} = \frac{\sqrt{2d \log (2m + 2) + \log (4  / \delta)}\sqrt{m}}{2 m_i}$.
\end{corollary}

In the next section, we design a framework using the error bound to guide the clustering of FL clients. 

\section{Proposed Methods}\label{sec:method}

In this section, we present our method {\algname} to optimize the collaboration structure under the guidance of Theorem \ref{thm:error}. We transform the error bound in Theorem \ref{thm:error} to an optimization objective, estimate client distribution differences without violating the privacy constraints, and design an efficient algorithm to optimize the collaboration structure with the awareness of both data quantity and distribution difference. Figure \ref{fig:overview} provides an overview of our proposed method. 

\subsection{{\algname} Objective}

We first transform the error bound to a practical optimization objective. We remove the non-optimizable $\epsilon_{i}(h_i^*)$, and replace the quantity-aware term $\phi_{|\cH|}(\valpha_i, \vbeta, m, \delta)$ and the pair-wise distribution difference $D(\cD_i, \cD_j)$ with empirical estimations. Finally, we combine error bounds for each client together to form a global objective for clustering. 

\begin{figure}
\vspace{2mm}
\begin{center}
\centerline{\includegraphics[width=1.0\columnwidth]{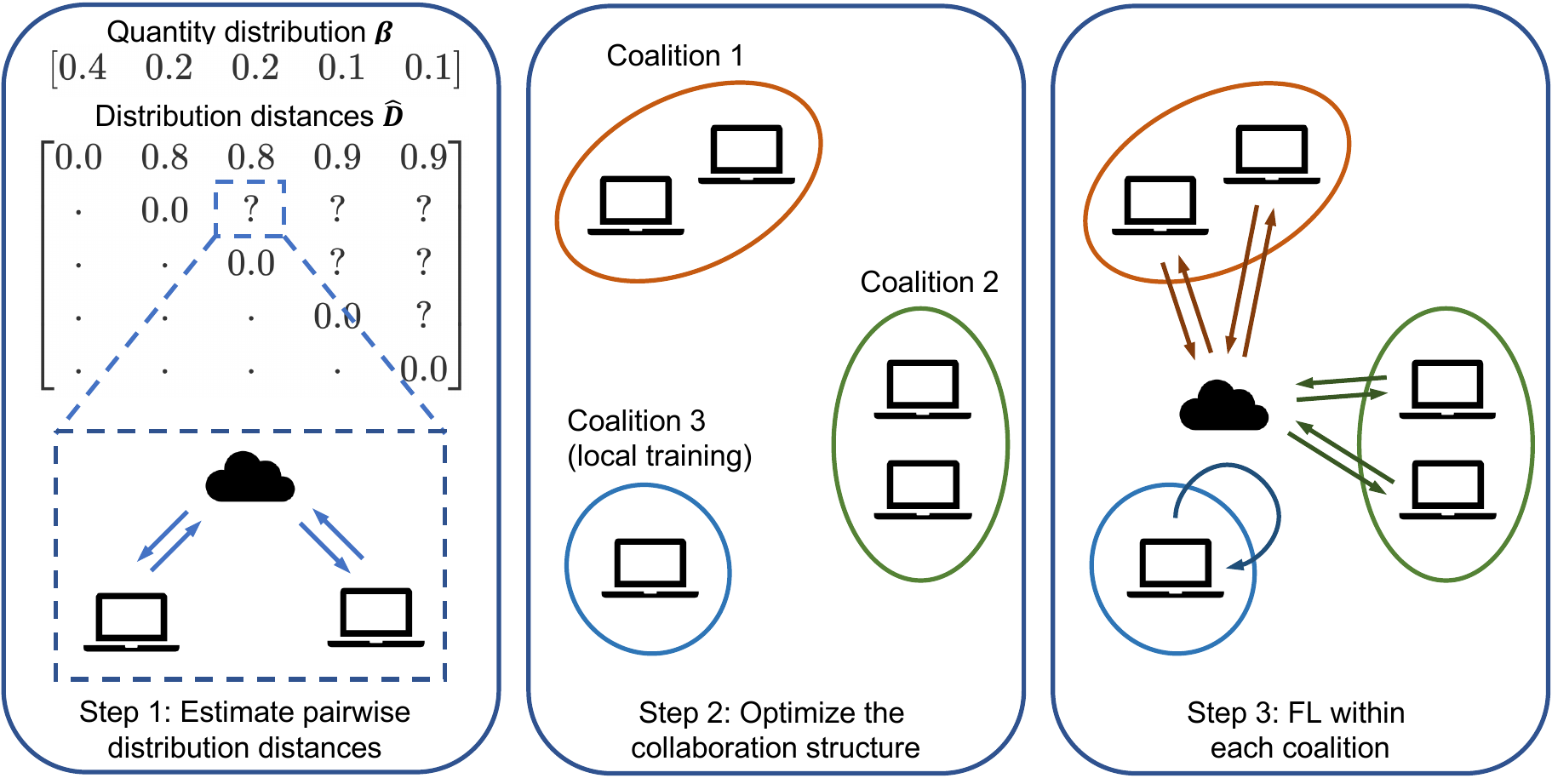}}
\caption{An overview of {\algname}}
\label{fig:overview}
\end{center}
\vspace{-5mm}
\end{figure}

\paragraph{Quantifying the Quantity-Aware Function}
The quantity-aware function $\phi_{|\cH|}(\valpha_i, \vbeta, m, \delta)$ indicates the influence of data quantity. However, it is related to the complexity of hypothesis space $\cH$, which can be hard to estimate accurately for neural networks. Inspired by earlier works on the model-complexity-based generalization bounds \cite{generalization1,generalization2,label_disc}, we treat the model capacity constant $C = \sqrt{2d \log (2m + 2) + \log (4 /\delta )}$ as a hyperparameter to tune. This gives the empirical quantity-aware function: 
\begin{align}
    \hat \phi(\valpha_i, \vbeta, m, C) = \frac{C}{\sqrt{m}} \cdot \sqrt{\sum_{j=1}^N \frac{\alpha_{ij}^2}{\beta_j}} 
\end{align}
Notice that $\vbeta$ and $m$ can be directly calculated with the quantities reported by each client directly. 

\paragraph{Quantifying the Distribution Differences}
In Theorem \ref{thm:error}, the distribution difference $D(\cD_i, \cD_j)$ is defined on two clients' underlying distributions $\cD_i, \cD_j$, which is typically not available in practice. Previous CFL algorithms \cite{cfl,fesem} usually rely on the similarity of parameters/gradients, which indirectly reflect the distribution difference of clients. These methods are less accurate in the estimation of distribution distance, due to the non-convexity and permutation invariance of neural networks~\cite{Wang2020Federated}. For example, even when we train two neural networks on two identical datasets, the parameters of two networks can vary significantly due to the differences in parameter initialization, the randomness of data loading, etc. 

In domain adaptation, when estimating the distribution distance between two domains (distributions), a common practice is to train a domain discriminator \cite{ben2010theory} to predict which domain a randomly drawn sample is from. However, traditional domain adaptation requires putting data from two domains together, which violates the privacy constraints of FL. Therefore, we design an algorithm to estimate pairwise distribution difference between two clients without sharing their data. Notice that different from domain adaptation, our goal is to estimate the distribution difference, rather than aligning two distributions. 

In particular, we use the $\cC$-divergence to quantify the distribution differences, i.e., 
\begin{align}
\begin{split}
    & D(\cD_i, \cD_j) = \max_{h \in \cH} \left| \epsilon_i(h) - \epsilon_j(h) \right| = \\
    & \max_{h \in \cH} \left| \bbE_{(\vx, \vy) \in \cD_i} \ell(h(\vx), \vy) - \bbE_{(\vx, \vy) \in \cD_j} \ell(h(\vx), \vy) \right| 
\end{split}
\end{align}
where $\ell$ is the 0-1 loss. We can further rewrite $f(\vx, \vy) = \ell(h(\vx), \vy)$ as a mapping $\cX \times \cY \to \{0, 1\}$. \add{With detailed derivation provided in Appendix \ref{eq:derive}}, the equation above can be transformed as
\begin{align*}
\begin{split}
    & \max_{f \in \cF} \left| \Pr_{(\vx, \vy) \in \cD_i} [f(\vx, \vy) = 1] +  \Pr_{(\vx, \vy) \in \cD_j} [f(\vx, \vy) = 0] - 1 \right| \\
    & = \max_{f \in \cF} \left| 2 \cdot \text{BalAcc} (f, \{\cD_i, 1\} \cup \{\cD_j, 0\}) - 1\right| 
\end{split}
\end{align*}
The equation above shows that we can train a \textit{client discriminator} $f \in \cF$ to predict 1, 0 on client $i, j$, respectively. The estimated distance is a simple function of the balanced accuracy (BalAcc) of the discriminator. Intuitively, when two distributions are distinctly different, a classifier will discriminate two distributions with BalAcc $\approx 100\%$, thus the distance $\approx 1$. 
Meanwhile, when two distributions are similar, the classifier cannot outperform random guessing, which results in BalAcc $\approx 50\%$ and thus the distance $\approx 0$. 

Notice that while our FL model takes features $\vx$ as input and predicts label, the client discriminator takes feature-label pairs $(\vx, \vy)$ as input and predicts sample origin. \add{By taking feature-label pairs as input, the estimated $\cC$-divergence can capture a wide range of distribution shifts, including feature shift (different $P(\vx)$), label shift (different $P(\vy)$), and concept shift (different $P(\vy | \vx)$).} In practice, we instantiate the client discriminator $f$ with a 2-layer neural network $f_{\vw}$ with parameters $\vw$, and train the client discriminator within FL framework, with pseudo-code in Algorithm \ref{alg:client_disc}. \add{By using light-weight client discriminator, estimating pairwise distribution differences is much more efficient than training FL models. We quantify and compare their computation and communication complexities in Appendix \ref{subsec:complexity}.} 
\setlength{\textfloatsep}{5mm}
\begin{algorithm}[tb]
    \caption{Training client discriminator \label{alg:client_disc}} 
    \begin{algorithmic}[1]
        \INPUT Clients $i, j$ with local datasets $\hat\cD_i, \hat\cD_j$, $m_{\text{train}}, \vw_S^0, T$
        \OUTPUT Distribution distance estimation $\hat{D}_{ij}$
        \STATE Train-valid split: $\hat{\cD}_i = \hat\cD_i^\text{train} \cup \hat\cD_i^{\text{valid}}, \hat{\cD}_j = \hat\cD_j^\text{train} \cup \hat\cD_j^{\text{valid}}$ with $|\hat\cD_i^\text{train}| = |\hat\cD_j^\text{train}| = m_{\text{train}}$
        \FOR{communication round $t = 1, \cdots, T$}
            \STATE Server sends $\vw_S^{t-1}$ to two clients
            \FOR{client $k \in \{i, j\}$ \textbf{in parallel}}
                \STATE Let client index $c = 1, 0$ for client $i, j$, respectively
                \STATE $\vw_k^t \leftarrow \text{LocalUpdate}(\vw_S^{t-1}, \{\hat\cD_i^\text{train}, c\})$
                \STATE Client sends $\vw_k^t$ to server
            \ENDFOR
            \STATE $\vw_S^t \leftarrow \frac{1}{2} (\vw_i^t + \vw_j^t)$
        \ENDFOR
        \STATE $\hat{D}_{ij} \leftarrow 2 \cdot \text{BalAcc}(f_{\vw_S^T}, \{\cD_i^{\text{valid}}, 1\} \cup \{\cD_j^{\text{valid}}, 0\}) - 1 $
    \end{algorithmic}
\end{algorithm}

\paragraph{Combining Error Bounds from All Clients}
Finally, we combine the error bounds of all clients to form the following objective function. Given a collaboration structure $\{\cC_1, \cdots, \cC_K\}$ with $K$ non-overlapping coalitions, where $K$ is an indeterminate number of coalitions, clients from the same coalition have the same collaboration vector $\valpha_i$ as defined in Eq. (\ref{eq:cfl_loss}) since they share the same global model. Here we define the \textit{collaboration matrix} $\vA = [\valpha_1^\top, \cdots, \valpha_N^\top]^\top$ as follows. 
\vspace{-1mm}
\begin{align}
    \vA_{ij} = \alpha_{ij} = \begin{cases}
        \frac{\beta_j}{\sum_{l \in \cC_k} \beta_l}, & \text{if } i \in \cC_k, j \in \cC_k, \exists k \\
        0, & \text{otherwise}
    \end{cases}
    \label{eq:collab_mat}
\end{align}
\vspace{-1mm}
Then, the {\algname} objective can be formulated as 
\begin{align}
\begin{split}
    \cL(\vA, \vbeta, m, \hat\vD) &= \sum_{i=1}^N \left( \frac{C}{\sqrt{m}} \sqrt{\sum_{j=1}^N \frac{\alpha_{ij}^2}{\beta_j}} + \sum_{j=1}^N \alpha_{ij} \hat{D}_{ij} \right) \\
    &= \frac{C}{\sqrt{m}} \sum_{i=1}^N \| \valpha_i \|_{\text{diag}(\vbeta)^{-1}}^2 + \vA \odot \hat\vD
\end{split}
\label{eq:collab_obj}
\end{align}
where $\odot$ is the element-wise product. In the next part, we propose an efficient optimizer to find a collaboration structure that minimizes the objective above. 

\subsection{{\algname} Optimizer}

Note that optimizing collaboration structure involves not only determining the objective but also how to optimize it. For example, while the objective is concise, constraints on $\vA$ in Eq. (\ref{eq:collab_mat}) make optimization challenging: the range of $\vA$ is discrete, and thus gradient descent cannot be directly used.

Therefore, we propose an efficient algorithm to solve the problem in discrete space. We optimize the coalition assignment $p(\cdot)$ which maps the client index to a coalition index (e.g., $p(1) = 2$ means assigning client 1 to coalition 2). We initialize the coalition assignment with local training, i.e., $p(i) = i, \forall i$, and iteratively assign clients to a new coalition that can further minimize the {\algname} objective in Eq (\ref{eq:collab_obj}). Algorithm \ref{alg:opt} gives the pseudo-code of the optimizer. 

\setlength{\textfloatsep}{5mm}
\begin{algorithm}[!t]
    \caption{{\algname} optimizer} \label{alg:opt}
    \begin{algorithmic}[1]
        \INPUT $\vbeta, m, \hat\vD$
        \OUTPUT Coalition assignment $p(\cdot)$
        \STATE Initialize $p(i) = i$ for all clients (local training)
        \WHILE{not converged}
            \FOR{client index $k$ in a permutation of $[1, \cdots, N]$}
                \STATE Evaluate the objective of Eq. (\ref{eq:collab_obj}) with the new collaboration structure after setting $p(k) = 1, \cdots, N$
                \STATE Update $p(k)$ to the coalition with lowest value of Eq. (\ref{eq:collab_obj})
            \ENDFOR
        \ENDWHILE
    \end{algorithmic}
\end{algorithm}

The optimizer guarantees to converge \add{to local optimum}, since the objective function has finite values and strictly decreases in each iteration. In practice, since greedy methods generally do not guarantee the \add{global} optimum, we re-run Algorithm \ref{alg:opt} multiple times with different random seeds to further refine the collaboration structure. Different from most CFL algorithms \cite{ifca,fesem}, where re-optimizing the collaboration structure requires re-training FL models and introduces large computation and communication costs, the collaboration optimization process of {\algname} is purely on the server and does not require training any ML model. As a result, our optimizer is very efficient and only takes a few seconds to run. 

\subsection{Training FL Models}

After solving the collaboration structure, {\algname} fixes the collaboration and trains FL models within each coalition separately. Notice that since the collaboration structure and the FL model are optimized independently, {\algname} can be seamlessly integrated with any GFL or PFL algorithms in this stage. 

\subsection{New Training Clients}

An additional advantage of {\algname} is that while typical cross-silo FL systems \cite{scaffold,mocha} are expensive to allow new clients to join after the training of FL models, our {\algname} framework allows new clients to join a cross-silo FL system without the need for re-clustering and re-training all FL models. In particular, {\algname} assigns new clients to existing coalitions that minimize the objective in Eq. (\ref{eq:collab_obj}) by estimating the distribution distance between the new client and existing clients, thus requiring only one coalition to fine-tune or re-train the FL model for each new client.

\section{Experiments}\label{sec:experiments}

\renewcommand{\meansd}[2]{$#1$ {\tiny($#2$)}}

In this section, we design experiments to answer the following research questions:
\begin{itemize}[itemsep=0pt,topsep=0pt] 
    \item \textbf{RQ1}: Can {\algname} alleviate negative transfer for both GFL and PFL? 
    \item \textbf{RQ2}: Can {\algname} provide better collaboration structures than previous CFL algorithms? 
    \item \textbf{RQ3} (hyperparameter): How do the choices of hyperparameter $C$ affect {\algname}? 
    \item \textbf{RQ4} (ablation study): How do different components contribute to the effectiveness of {\algname}?
    \item \textbf{RQ5}: Can {\algname} utilize new training clients? (see Appendix \ref{subsec:newclient})
    \item \textbf{RQ6} (convergence): Does {\algname} optimizer converge efficiently and effectively? (see Appendix \ref{subsec:converge})
\end{itemize}

\begin{table*}
\vspace{-2mm}
\setlength{\tabcolsep}{0.9mm}{
\caption{Alleviating negative transfer of base GFL and PFL algorithms with different models, datasets, and types of non-IIDness, where we report the mean and standard deviation for each evaluation metric in percentage ($\%$) after five runs. \label{tab:exp_nt}}
\vspace{1mm}
\begin{center}
\begin{small}
\begin{tabular}{l|ccc|ccc|ccc}
\toprule
\multirow{2}{*}{Method} & \multicolumn{3}{c|}{Label Shift (FashionMNIST)} & \multicolumn{3}{c|}{Feature Shift (CIFAR-10)} & \multicolumn{3}{c}{Concept Shift (CIFAR-100)} \\
 & Acc $\uparrow$ & IPR $\uparrow$ & RSD $\downarrow$ & Acc $\uparrow$ & IPR $\uparrow$ & RSD $\downarrow$ & Acc $\uparrow$ & IPR $\uparrow$ & RSD $\downarrow$ \\
\midrule
Local Train     & \meansd{86.05}{0.28}  & - & - 
                & \meansd{38.65}{0.44} & - & - 
                & \meansd{29.82}{0.56} & - & -  \\
\midrule
FedAvg          & \meansd{46.64}{0.12}  & \meansd{46.00}{2.24}  & \meansd{41.03}{0.24} 
                & \meansd{44.31}{0.98}  & \meansd{86.00}{4.18}  & \meansd{4.62}{0.58} 
                & \meansd{26.62}{0.12}  & \meansd{50.00}{0.00}  & \meansd{11.54}{0.45}  \\
\ +{\algname}   & \meansd{92.45}{0.07}  & \meansd{100.00}{0.00} & \meansd{5.99}{0.41} 
                & \meansd{52.61}{0.60}  & \meansd{100.00}{0.00} & \meansd{3.30}{0.63}
                & \meansd{40.94}{0.22}  & \meansd{100.00}{0.00} & \meansd{2.78}{0.30}   \\
\midrule
FedProx         & \meansd{46.70}{0.08}  & \meansd{45.00}{5.00}  & \meansd{41.09}{0.29}
                & \meansd{44.45}{0.58}  & \meansd{87.00}{4.47}  & \meansd{4.74}{0.56} 
                & \meansd{26.78}{0.14}  & \meansd{50.00}{0.00}  & \meansd{11.66}{0.36}   \\
\ +{\algname}   & \meansd{92.39}{0.15}  & \meansd{100.00}{0.00} & \meansd{6.02}{0.37} 
                & \meansd{52.73}{0.64}  & \meansd{100.00}{0.00}  & \meansd{3.16}{0.61} 
                & \meansd{40.99}{0.17}  & \meansd{100.00}{0.00} & \meansd{2.79}{0.34}\\
\midrule
FedNova         & \meansd{75.92}{1.14}  & \meansd{45.00}{3.54}  & \meansd{12.38}{1.25} 
                & \meansd{46.98}{0.57}  & \meansd{99.00}{2.24}  & \meansd{3.42}{0.22} 
                & \meansd{26.46}{0.13}  & \meansd{50.00}{0.00}  & \meansd{10.57}{0.32} \\
\ +{\algname}   & \meansd{92.47}{0.13}  & \meansd{100.00}{0.00} & \meansd{5.97}{0.39} 
                & \meansd{52.72}{0.57}  & \meansd{100.00}{0.00}  & \meansd{3.18}{0.63} 
                & \meansd{40.92}{0.36}  & \meansd{100.00}{0.00}  & \meansd{2.75}{0.43}\\
\midrule
Finetune        & \meansd{67.32}{3.17}  & \meansd{48.00}{2.74}  & \meansd{22.97}{2.82} 
                & \meansd{44.17}{0.99}  & \meansd{82.00}{2.74}  & \meansd{5.14}{0.32} 
                & \meansd{33.30}{4.79}  & \meansd{50.00}{0.00}  & \meansd{13.95}{0.57}\\
\ +{\algname}   & \meansd{92.57}{0.15}  & \meansd{99.00}{2.24}  & \meansd{6.07}{0.30} 
                & \meansd{51.53}{0.61}  & \meansd{100.00}{0.00}  & \meansd{2.92}{0.46} 
                & \meansd{40.94}{2.36}  & \meansd{100.00}{0.00}  & \meansd{2.54}{0.30} \\
\midrule
Per-FedAvg      & \meansd{51.13}{4.10}  & \meansd{49.00}{2.24}  & \meansd{37.35}{4.15}
                & \meansd{43.78}{0.69}  & \meansd{83.00}{9.08}  & \meansd{4.74}{0.65} 
                & \meansd{27.39}{0.24}  & \meansd{50.00}{0.00}  & \meansd{12.24}{0.46}  \\
\ +{\algname}   & \meansd{92.16}{0.25}  & \meansd{97.00}{6.71}  & \meansd{6.00}{0.25}
                & \meansd{52.64}{0.45}  & \meansd{100.00}{0.00}  & \meansd{3.03}{0.30} 
                & \meansd{41.04}{0.26}  & \meansd{100.00}{0.00} & \meansd{2.85}{0.49}   \\
\midrule
pFedMe          & \meansd{55.31}{3.45}  & \meansd{47.00}{4.47}  & \meansd{33.71}{3.11} 
                & \meansd{39.74}{0.85}  & \meansd{60.00}{12.25}  & \meansd{4.81}{0.74} 
                & \meansd{27.04}{0.39}  & \meansd{48.00}{2.74}  & \meansd{10.39}{0.47}\\
\ +{\algname}   & \meansd{92.18}{0.43}  & \meansd{99.00}{2.24}  & \meansd{6.40}{0.81} 
                & \meansd{47.20}{1.29}  & \meansd{97.00}{2.74}  & \meansd{3.02}{0.30} 
                & \meansd{37.47}{0.31}  & \meansd{100.00}{0.00}  & \meansd{3.04}{0.23}\\
\midrule
Ditto           & \meansd{68.73}{1.40}  & \meansd{48.00}{2.74}  & \meansd{20.29}{2.06} 
                & \meansd{47.04}{0.30}  & \meansd{97.00}{2.74}  & \meansd{3.85}{0.35} 
                & \meansd{32.50}{0.40}  & \meansd{50.00}{0.00}  & \meansd{12.22}{0.36}\\
\ +{\algname}   & \meansd{92.55}{0.08}  & \meansd{99.00}{2.24}  & \meansd{6.11}{0.30}
                & \meansd{50.97}{0.75}  & \meansd{99.00}{2.24}  & \meansd{3.38}{1.55} 
                & \meansd{40.33}{0.33}  & \meansd{100.00}{0.00}  & \meansd{2.16}{0.30}\\
\bottomrule
\end{tabular}
\end{small}
\end{center}
}
\vspace{-2mm}
\end{table*}


\subsection{Setup}
\label{sec:exp_setup}

\paragraph{Models and Datasets}
We evaluate our framework on three models and datasets: we train a 3-layer MLP for FashionMNIST \cite{fmnist}, a 5-layer CNN for CIFAR-10 \cite{cifar}, and an ImageNet pre-trained ResNet-18 \cite{resnet} for CIFAR-100 (with 20 coarse labels). We simulate three typical scenarios of non-IIDness \cite{advance} on three datasets respectively, to show that our algorithm can handle a wide range of non-IIDness. For all scenarios, we simulate 20 clients with four types. 
\begin{itemize}[itemsep=0pt,topsep=0pt]
    \item \textbf{Label shift} \cite{cfl_converge}. Each client has a different label distribution. Figure \ref{fig:labelshift} visualizes the label and quantity distribution for each client. Different from Dirichlet partition \cite{dirichlet}, where the distribution distance between any two clients has the same expectation, we create multiple levels of distribution distances. For example, client 0's label distribution is most close to clients 1-4, less close to clients 5-9, and very distinct to clients 10-19. 
    
    \item \textbf{Feature shift} \cite{ifca}. Each client's image is rotated for a given angle: $+25^{\circ}$, $-25^{\circ}$, $+155^{\circ}$, $-155^{\circ}$ for clients 0-4, 5-9, 10-14, and 15-19, respectively. Multiple levels of distribution distances also exist in this scenario: client 0's images have $0^{\circ}$ angle difference from client 1-4, $50^{\circ}$ from client 5-9, $130^{\circ}$ from client 10-14, and $180^{\circ}$ from client 15-19. 
    
    \item \textbf{Concept shift} \cite{cfl}. Each client's label indices are permuted with the order given in Figure \ref{fig:conceptshift}. Similar multiple levels of distribution distances are constructed: client 0 has all labels aligned with clients 1-4, 14 labels aligned with clients 5-9, and no label aligned with clients 10-19. 
\end{itemize}
To simulate quantity shift while remaining explainability, we let clients 0-9 be ``large'' clients with more data, and clients 10-19 be ``small'' clients with less data. As a result, the large clients are more picky, and perform the best when they only collaborate with the same type of client (e.g., client 0 performs the best within a coalition of 0-4). However, small clients will prefer larger coalitions (e.g., client 10 performs the best within a coalition of 10-19). 

\begin{figure}
\begin{center}
\includegraphics[width=0.8\columnwidth]{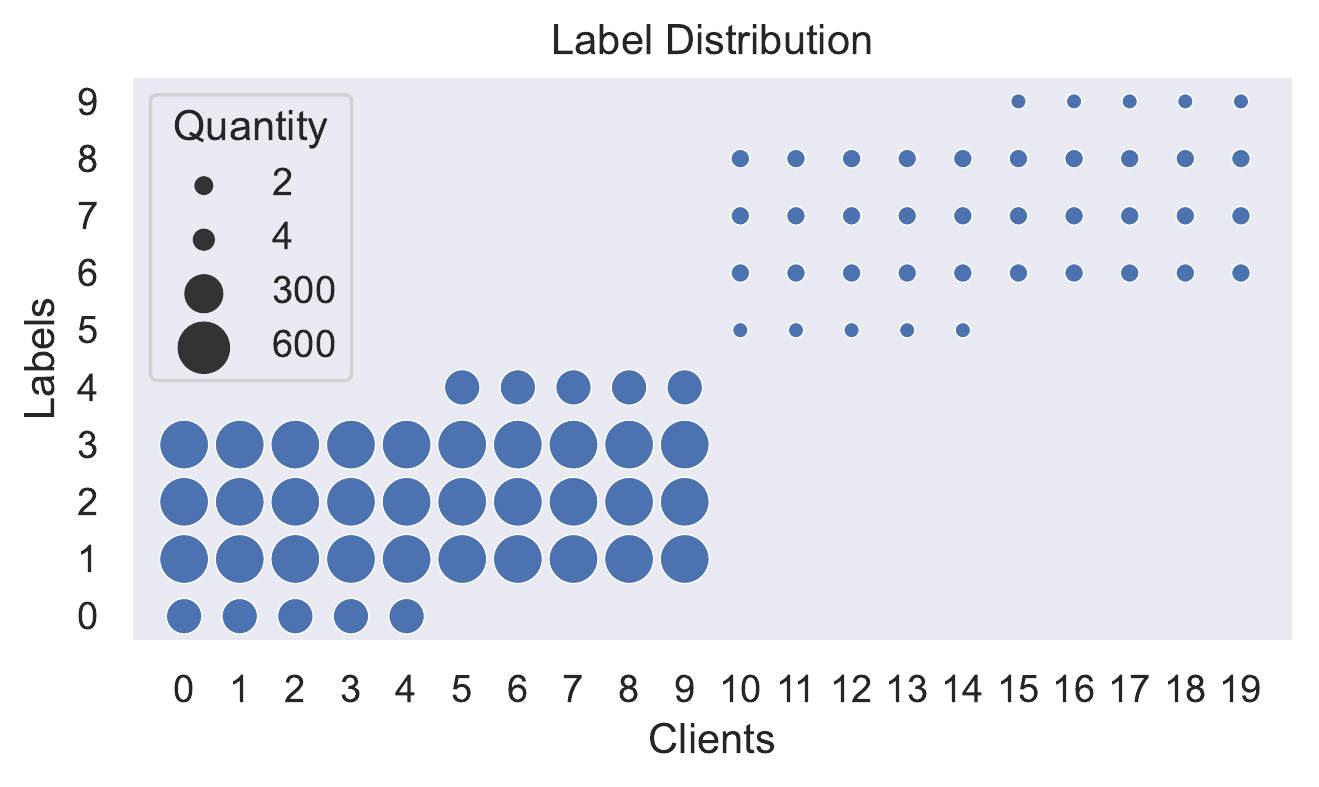}
\caption{Label and quantity distributions for label shift scenario. }
\label{fig:labelshift}
\end{center}
\end{figure}

\begin{figure}
\begin{center}
\includegraphics[width=0.85\columnwidth]{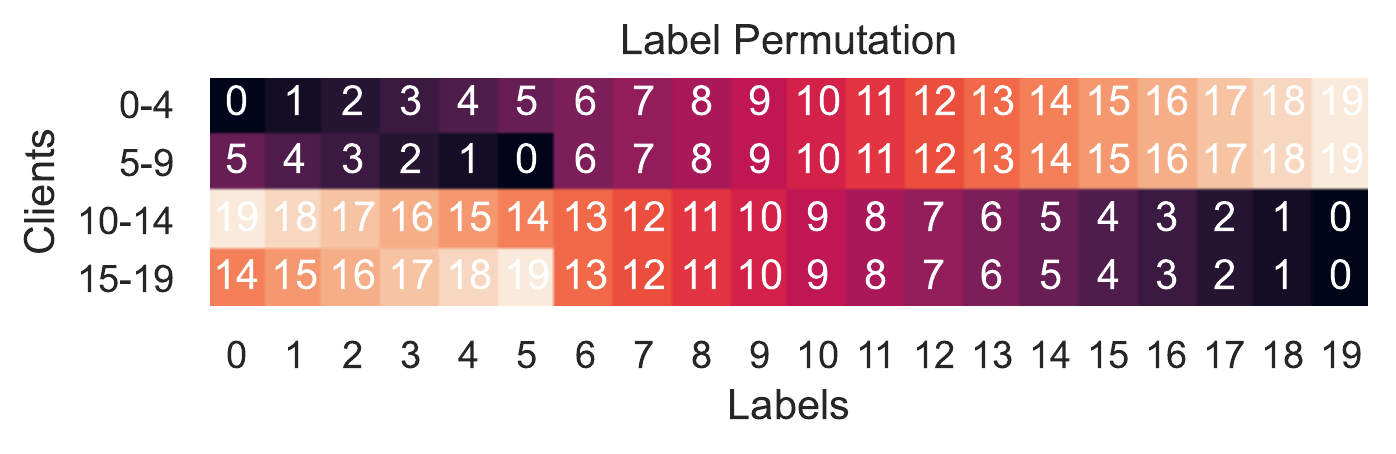}
\caption{Label permutation for concept shift scenario. }
\label{fig:conceptshift}
\end{center}
\end{figure}

\paragraph{Metrics}
To comprehensively evaluate the FL algorithms, besides the accuracy score (Acc), we use \textit{incentivized participation rate} (IPR) \cite{incfl} to evaluate how many clients get accuracy gains compared to local training, and \textit{reward standard deviation} (RSD) to evaluate the fairness of accuracy gains. Both metrics are defined with local model $\hat{h}_i^{\text{local}}$ and FL model $\hat{h}_i^{\text{FL}}$. 
\begin{align}
    \text{IPR} &= \frac{1}{N} \sum_{i=1}^N \bbI\{\text{acc}(\hat{h}_i^{\text{FL}}) - \text{acc}(\hat{h}_i^{\text{local}}) > 0\} \\
    \text{RSD} &= \text{SD}(\{\text{acc}(\hat{h}_i^{\text{FL}}) - \text{acc}(\hat{h}_i^{\text{local}})\}_{i=1}^N)
\end{align}
where $\text{SD}$ is the standard deviation. In an ideal FL system, all clients can get similar accuracy gains, which indicates a large IPR and small RSD. 

For all three datasets, we use a light-weight two-layer MLP as the client discriminator to estimate pairwise distribution distances. For CIFAR-10/CIFAR-100, we use an ImageNet pre-trained ResNet-18 to encode the raw image to 512 dimensions as a pre-processing step, before feeding it into the client discriminator. Notice that since the parameters of the ResNet-18 encoder is not trained or transmitted, it does not introduce any additional communication cost. 


\begin{table*}
\vspace{-2mm}
\setlength{\tabcolsep}{0.9mm}{
\caption{Comparison with Clustered FL \label{tab:exp_compare}}
\begin{center}
\vskip 2mm
\begin{small}
\begin{tabular}{l|ccc|ccc|ccc}
\toprule
\multirow{2}{*}{Method} & \multicolumn{3}{c|}{Label Shift (FashionMNIST)} & \multicolumn{3}{c|}{Feature Shift (CIFAR-10)} & \multicolumn{3}{c}{Concept Shift (CIFAR-100)} \\
 & Acc $\uparrow$ & IPR $\uparrow$ & RSD $\downarrow$ & Acc $\uparrow$ & IPR $\uparrow$ & RSD $\downarrow$ & Acc $\uparrow$ & IPR $\uparrow$ & RSD $\downarrow$ \\
\midrule
IFCA        & \meansd{91.49}{0.61}  & \meansd{95.00}{5.00}  & \meansd{5.62}{0.54} 
            & \meansd{49.78}{1.01}  & \meansd{100.00}{0.00}  & \meansd{3.13}{0.52} 
            & \meansd{30.74}{4.46}  & \meansd{60.00}{22.36} & \meansd{11.28}{5.04} \\
FedCluster  & \meansd{92.07}{0.47}  & \meansd{95.00}{7.07}  & \meansd{6.14}{0.49} 
            & \meansd{44.86}{1.90}  & \meansd{79.00}{17.10}  & \meansd{5.64}{1.81} 
            & \meansd{29.23}{2.18}  & \meansd{62.00}{12.55} & \meansd{9.55}{0.69} \\
FeSEM       & \meansd{56.79}{6.71}  & \meansd{45.00}{11.18}  & \meansd{36.12}{2.08} 
            & \meansd{42.73}{0.37}  & \meansd{82.00}{5.70}  & \meansd{4.10}{0.62}  
            & \meansd{31.92}{3.12}  & \meansd{72.00}{12.55}  & \meansd{9.81}{1.77} \\
KMeans      & \meansd{69.30}{0.81}  & \meansd{72.00}{2.74}  & \meansd{35.87}{1.22} 
            & \meansd{48.61}{1.15}  & \meansd{96.00}{4.18}  & \meansd{4.54}{0.74} 
            & \meansd{34.24}{3.01}  & \meansd{85.00}{13.69}  & \meansd{6.47}{3.06}\\
{\algname}  & \meansd{92.45}{0.07}  & \meansd{100.00}{0.00} & \meansd{5.99}{0.41} 
            & \meansd{52.61}{0.60}  & \meansd{100.00}{0.00} & \meansd{3.30}{0.63} 
            & \meansd{40.94}{0.22}  & \meansd{100.00}{0.00} & \meansd{2.78}{0.30}\\
\bottomrule
\end{tabular}
\end{small}
\end{center}
}
\end{table*}
\subsection{Alleviating Negative Transfer (RQ1)}
\label{subsec:alleviate}

We first show that while GFL and PFL algorithms suffer from negative transfer, after integrated with {\algname}, their negative transfer can be alleviated. We consider a wide range of SOTA GFL and PFL algorithms. For GFL, besides FedAvg \cite{fedavg}, we also compare to FedProx \cite{fedprox} (for better stability to non-IIDness) and FedNova \cite{fednova} (for more consistent objective under quantity shift). For PFL, we include Finetune (where each client locally finetunes the FedAvg model), a meta-learning-based method Per-FedAvg \cite{per-fedavg}, a regularization-based method pFedMe \cite{pfedme}, and a fair-and-robust method Ditto \cite{ditto}. 

We report the results in Table \ref{tab:exp_nt}. Across datasets, models and types of non-IIDness, our proposed {\algname} strongly enhances the performance of all seven base FL algorithms in terms of accuracy, IPR and fairness (RSD). In the label shift and concept shift scenarios, all the base GFL and PFL algorithms strongly suffer from negative transfer: more than half of the clients (mostly the small clients) receive a FL model worse than local model. Although PFLs introduce accuracy gain compared to FedAvg, they do not solve the negative transfer problem since small clients still do not benefit from FL. However, when combined with our {\algname}, all base FL algorithms can reach a near $100\%$ IPR with much better accuracy and reward fairness. 

In the feature shift scenarios, since rotation is a mild kind of non-IIDness also used for data augmentation, base FL algorithms suffer less from negative transfer compared to the other two scenarios: all the base FL algorithms get accuracy gain in average. Our {\algname} framework can further boost these FL algorithms to the next level, also reach a near $100\%$ IPR with significantly better accuracy and reward fairness. 

It is also interesting to notice that after combining with our {\algname} framework, four PFL algorithms have limited or no accuracy gain compared to GFL algorithms. This enlightens us that ``who to collaborate'' may be more important than ``how to collaborate'', and should be considered first. 


\subsection{Comparison to other CFL Algorithms (RQ2)}
\label{sec:exp_compare}

In this part, we compare our {\algname} algorithm (combined with FedAvg) to baseline CFL algorithms, including one loss-based algorithm IFCA \cite{ifca}, one gradient-based algorithm FedCluster \cite{cfl}, and two parameter-based algorithms FeSEM \cite{fesem} and KMeans \cite{kmeans}. 
We report the results in Table \ref{tab:exp_compare}. Across all scenarios, {\algname} has the highest accuracy and IPR, with RSD among the lowest. Besides numerical results, we further study why {\algname} has better performance than baseline CFL methods. 

\paragraph{Quantity Awareness}
While {\algname} explicitly uses the quantity distribution $\vbeta$ during collaboration optimization, all four baseline CFL algorithms cannot utilize the quantity information. For example, IFCA uses training losses to choose the model (cluster), which is not sensitive to the quantities. Therefore, it usually results in two clusters: 0-9 and 10-19, without further splitting the ``large'' clients. 

\begin{figure}
\begin{center}
\centerline{\includegraphics[width=1.0\columnwidth]{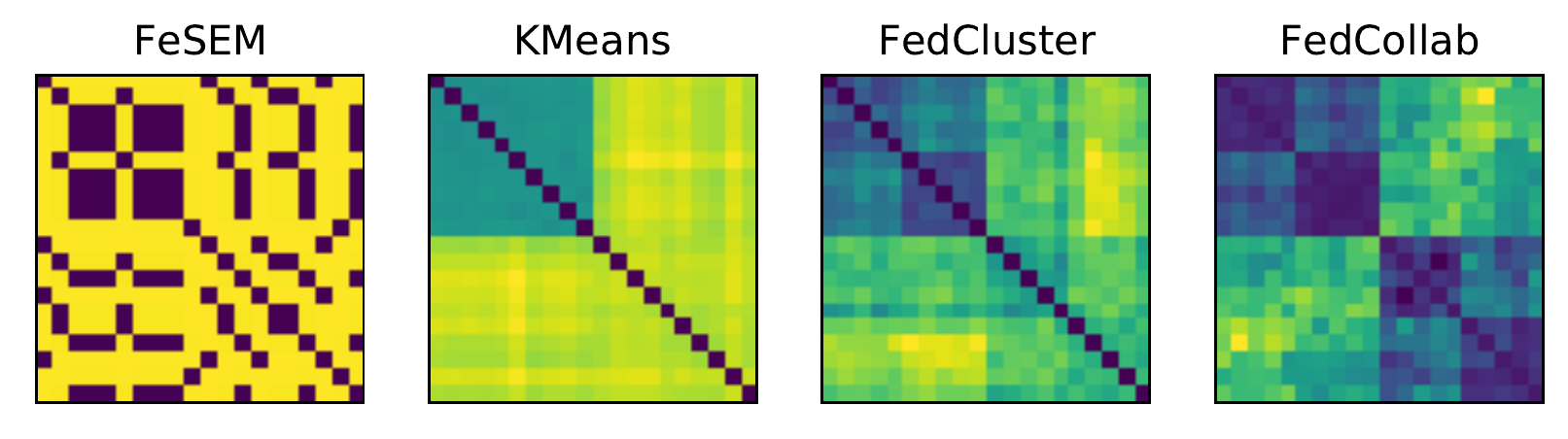}}
\caption{Client distance matrices on CIFAR-10 with feature shift. }
\label{fig:distance}
\end{center}
\end{figure}

\paragraph{Distribution Distances}
Apart from the quantity awareness, our {\algname} framework estimates high-quality distribution distances. Notice that FeSEM and KMeans rely on the distance between model parameters, while FedCluster relies on gradient similarity matrix $\vS$. We visualize the distance matrix of {\algname} and these baselines in Figure \ref{fig:distance} (for FedCluster we show $1 - \vS$). It can be seen that the distance matrix of FeSEM is highly random depending on the initialization. KMeans gives some meaningful estimations, but the distance between two clients with the same underlying distribution is still high. While FedCluster gives the best estimation among baselines, the estimated distribution distance of {\algname} clearly reveals the multi-level distribution distances we construct. 

\add{Besides the performance, we also point out that while IFCA, FedCluster and FeSEM perform clustering \textit{during FL}, KMeans and {\algname} perform clustering \textit{before FL}. We compare these two types of CFL in Appendix \ref{subsec:cluster_fl}.} 


\subsection{Effects of Hyperparameter (RQ3)}
\label{sec:exp_hyperparameter}

Our algorithm has a hyperparameter $C$ that balances generalization error and dataset shift. We study the effect of $C$ with results shown in Figure \ref{fig:hyper}. When $C = 6, 8, 10$, {\algname} gives the same collaboration structure with the highest accuracy. When we decrease $C$, the solved collaboration structure changes from coarse to fine, and finally to local training when $C = 0$. On the other hand, when $C$ goes to infinity, the solved collaboration structure changes to global training, which suffers from negative transfer. 

\begin{figure}
\begin{center}
\centerline{\includegraphics[width=0.8\columnwidth]{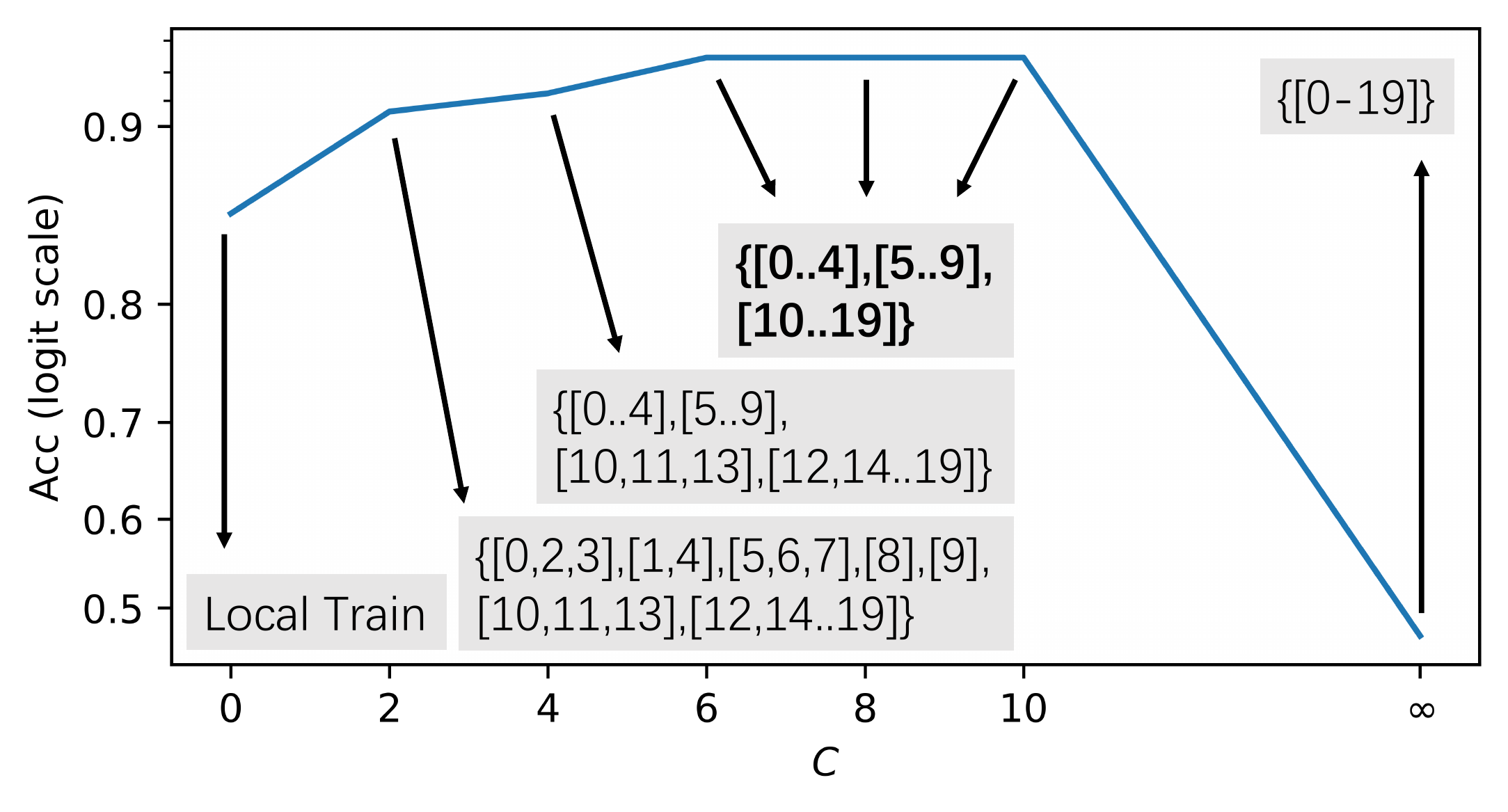}}
\caption{Effects of $C$ on FashionMNIST with label shift. }
\label{fig:hyper}
\end{center}
\vspace{-2mm}
\end{figure}


\subsection{Ablation Study (RQ4)}
\label{sec:exp_ablation}

In this part, we show that both distribution distances and quantity contribute to the optimization of collaboration structure. To this end, we consider two variants of {\algname}. With dataset untouched, ``ignore quantities'' replaces the real quantity distribution $\vbeta$ with a uniform vector $\frac{1}{N} \boldsymbol{1}$, while ``ignore distances'' replaces the non-diagonal elements in the estimated distribution distance matrix $\hat\vD$ with their average. 

Table \ref{tab:exp_ablation} summarizes the results of the ablation study. When ignoring distances, we observe that {\algname} assigns clients with no overlapping labels to the same coalition, which results in worse performance. When ignoring quantities, we observe that {\algname} forms multiple coalitions for small clients, instead of a large coalition for clients 10-19. Therefore, small clients get smaller performance gains compared to the original {\algname}. 

\begin{table}[H]
\caption{Ablation study on FashionMNIST with label shift \label{tab:exp_ablation}}
\setlength{\tabcolsep}{0.9mm}{
\begin{center}
\begin{small}
\begin{tabular}{l|ccc}
\toprule
Method & Acc $\uparrow$ & IPR $\uparrow$ & RSD $\downarrow$ \\
\midrule
{\algname}          & \meansd{92.45}{0.07}  & \meansd{100.00}{0.00} & \meansd{5.99}{0.41} \\
Ignore quantities   & \meansd{90.31}{0.11}  & \meansd{94.00}{5.48}  & \meansd{4.30}{0.46} \\
Ignore distances    & \meansd{67.79}{0.89}  & \meansd{19.00}{4.18}  & \meansd{18.97}{0.83} \\
\bottomrule
\end{tabular}
\end{small}
\end{center}
}
\end{table}

\section{Conclusion}\label{sec:conclusion}

We present {\algname}, a CFL framework that alleviates negative transfer in FL. Inspired by our derived generalization error bound for FL clients,  {\algname} utilizes both quantity and distribution distance information to optimize the collaboration structure among clients. Extensive experiments demonstrate that {\algname} can boost the accuracy, incentivized participation rate and fairness of a wide range of GFL and PFL algorithms and a variety of non-IIDness. Moreover, {\algname} significantly outperforms state-of-the-art clustered FL algorithms in optimizing the collaboration structure among clients.

\section*{Acknowledgement}

This work is supported by National Science Foundation under Award No. IIS-1947203, IIS-2117902, IIS-2137468, and Agriculture and Food Research Initiative (AFRI) grant no. 2020-67021-32799/project accession no.1024178 from the USDA National Institute of Food and Agriculture. The views and conclusions are those of the authors and should not be interpreted as representing the official policies of the funding agencies or the government.

\bibliography{references}
\bibliographystyle{icml2023}

\newpage
\appendix
\onecolumn

\section{Proofs}

\setcounter{mainsection}{3}

\subsection{Proof of Theorem \ref{thm:error}}\label{sec:proof_of_theorem_3_1}

In this part, we give the proof of Theorem \ref{thm:error2}. We first formally define the quantity-aware function $\phi_{|\cH|}(\valpha_i, \vbeta, m, \delta)$ in Definition \ref{def:gen2} and the distribution difference term $D(\cD_i, \cD_j)$ in Definition \ref{def:dif2}

\subsubsection{Quantity-aware function}

\begin{maindefinition}[Quantity-aware function] \label{def:gen2}
    For a given hypothesis space $\cH$, fixed combination weights $\valpha_i$, quantity distribution $\vbeta$, total quantity $m$, for any $\delta \in (0, 1)$, with probability at least $1 - \delta$ (over the choice of the samples), the following holds for all $h \in \cH$, 
    \begin{align}
        | \hat\epsilon_{\valpha_i}(h) - \epsilon_{\valpha_i}(h) | \leq \phi_{|\cH|}(\valpha_i, \vbeta, m, \delta)
    \end{align}
\end{maindefinition}

\begin{remark}
    Definition \ref{def:gen2} is an abstract form of the difference between $\epsilon_{\valpha_i}(h)$, the expected loss on the mixture population distribution $\sum_{j=1}^N \alpha_{ij} \cD_j$, and $\hat\epsilon_{\valpha_i}(h)$, the empirical risk on finite samples drawn from the mixture empirical distribution $\sum_{j=1}^N \alpha_{ij} \hat\cD_j$. It can be instantiated with traditional generalization error bounds. In the main text we give an example with VC dimension \cite{ben2010theory}: 
    \begin{align}
        \phi_{|\cH|}^{\text{VC}}(\valpha_i, \vbeta, m, \delta) = \sqrt{\left( \sum_{j=1}^N \frac{\alpha_{ij}^2}{\vbeta_j} \right) \left( \frac{2d \log (2m + 2) + \log(4 / \delta)}{m} \right)} \label{def:vc2}
    \end{align}
    Another choice is using weighted Rademacher complexity \cite{weighted_rad}, which gives a similar form of the bound. 
    \begin{align}
        \phi_{|\cH|}^{\text{Rad}}(\valpha_i, \vbeta, m, \delta) = \hat{R}_{\valpha_i}(\cH) + 3 \sqrt{\frac{m}{2}\left( \max_{1 \leq j \leq N}\frac{\alpha_{ij}}{m_i}\right)^2 \log \left( \frac{2}{\delta}\right)}
    \end{align}
    where 
    \begin{align}
    \hat{R}_{\valpha_i}(\cH) = \bbE_{\vsigma \in \{\pm 1\}^m} \sup_{h \in \cH} 2\sum_{j=1}^N \frac{\alpha_{ij}}{m_j}\sum_{k=1}^{m_i} \sigma_{j, k} \ell(h(\vx_k^{(j)}), \vy_k^{(j)})
    \end{align}
    It can be transformed into a similar form. Denote $\hat{R}_{j}(\cH) = \bbE_{\vsigma_j \in \{\pm 1\}^{m_j}} \sup_{h \in \cH} 2 \sum_{k=1}^{m_i} \sigma_{j, k} \ell(h(\vx_k^{(j)}), \vy_k^{(j)})$ be the empirical Rademacher complexity of client $j$ with order $\cO(\frac{1}{\sqrt{m_j}})$ \cite{shalev2014understanding}, we have
    \begin{align}
        \hat{R}_{\valpha_i}(\cH) &\leq  \sum_{j=1}^N \alpha_{ij}\hat{R}_{j}(\cH) 
        \leq \sqrt{N \sum_{j=1}^N \alpha_{ij}^2(\hat{R}_{j}(\cH))^2} 
        \leq \sqrt{N \sum_{j=1}^N \alpha_{ij}^2 \cdot \left(\frac{C}{\sqrt{m_j}}\right)^2 }
        = \sqrt{\left(\sum_{j=1}^N \frac{\alpha_{ij}^2}{\vbeta_j} \right) \cdot \frac{N C^2}{m}} , \quad \exists C > 0 
    \end{align}
\end{remark}

\subsubsection{Distribution differences}

\begin{maindefinition}[Distribution differences] \label{def:dif2}
    For a given hypothesis space $\cH$, two distributions $\cD_i, \cD_j$, the following holds for all $h \in \cH$, 
    \begin{align}
        | \epsilon_i(h) - \epsilon_j(h) | \leq D(\cD_i, \cD_j)
    \end{align}
\end{maindefinition}

\begin{remark}
    Definition \ref{def:dif2} also can be instantiated with different distribution distances. In the main text we focus on $\cC$-divergence \cite{label_disc,c_divergence}, which utilizes both feature and label information. 
    \begin{align}
        D^{\cC}(\cD_i, \cD_j) = \max_{h \in \cH} | \epsilon_i(h) - \epsilon_j(h) | 
    \end{align}
    Another common choice is using $\cH \Delta \cH$-distance \cite{ben2010theory}. Denote $\cX_i, \cX_j$ as the marginal feature distributions of $\cD_i, \cD_j$, respectively, 
    \begin{align}
        D^{\cH\Delta\cH}(\cD_i, \cD_j) = \frac{1}{2} d_{\cH \Delta \cH} (\cX_i, \cX_j) + \lambda_{ij}
    \end{align}
    where $\lambda_{ij} = \min_{h \in \cH} (\epsilon_i(h) + \epsilon_j(h) )$ is assumed to be small and $d_{\cH \Delta \cH} (\cX_i, \cX_j)$ can be estimated with a client discriminator using only feature as input. 
\end{remark}

\subsubsection{Error Upper Bound}

\begin{lemma}[Error decomposition] \label{lemma:err_decomp}
    For all $h \in \cH$, denote $\epsilon_{\valpha_i}(h) = \sum_{j=1}^N \alpha_{ij} \epsilon_{\alpha_j}(h)$, 
    \begin{align*}
        \left| \epsilon_i(h) - \epsilon_{\valpha_i}(h) \right| 
        &= \left|\sum_{j=1}^N \alpha_{ij} \epsilon_i(h) - \sum_{j=1}^N \alpha_{ij} \epsilon_j(h) \right|  \\
        &\leq \sum_{j\neq i} \alpha_{ij}\left| \epsilon_i(h) - \epsilon_j(h) \right| \\
        &\leq \sum_{j\neq i} \alpha_{ij} D(\cD_i, \cD_j) \tag{Definition \ref{def:dif2}} 
    \end{align*}
\end{lemma}

\begin{maintheorem}\label{thm:error2}
Let $\hat{h}_{\valpha_i}$ be the empirical risk minimizer defined in Eq. (\ref{eq:cfl_loss}) and $h_i^*$ be client $i$'s expected risk minimizer. For any $\delta \in (0, \frac{1}{2})$, with probability at least $1 - 2 \delta$, the following holds
\begin{align}
    \epsilon_i(\hat{h}_{\valpha_i}) \leq \epsilon_i( h_i^* ) & + 2 \phi_{|\cH|}(\valpha_i, \vbeta, m, \delta) + 2 \sum_{j \neq i} \alpha_{ij} D(\cD_i, \cD_j)
\end{align}
\end{maintheorem}

\begin{proof}
\begin{align*}
        \epsilon_i(\hat h_{\valpha_i}) 
        &\leq \epsilon_{\valpha_i}(\hat h_{\valpha_i}) + \left| \epsilon_{\valpha_i}(\hat h_{\valpha_i}) - \epsilon_i(\hat h_{\valpha_i}) \right| \\
        &\leq \epsilon_{\valpha_i}(\hat h_{\valpha_i}) + \sum_{j \neq i} \alpha_{ij} D(\cD_i, \cD_j) \tag{Lemma \ref{lemma:err_decomp}} \\
        &\leq \hat\epsilon_{\valpha_i}(\hat h_{\valpha_i}) + \left| \epsilon_{\valpha_i}(\hat h_{\valpha_i}) - \hat\epsilon_{\valpha_i}(\hat h_{\valpha_i}) \right| + \sum_{j \neq i} \alpha_{ij} D(\cD_i, \cD_j) \\
        &\leq \hat\epsilon_{\valpha_i}(\hat h_{\valpha_i}) + \phi_{|\cH|}(\valpha_i, \vbeta, m, \delta) + \sum_{j \neq i} \alpha_{ij} D(\cD_i, \cD_j) \tag{Definition \ref{def:gen2}, hold with probability $1 - \delta$}\\
        &\leq \hat\epsilon_{\valpha_i}(h_i^*) + \phi_{|\cH|}(\valpha_i, \vbeta, m, \delta) + \sum_{j \neq i} \alpha_{ij} D(\cD_i, \cD_j) \tag{Empirical Minimizer} \\
        &\leq \epsilon_{\valpha_i}(h_i^*) + \left| \epsilon_{\valpha_i}(h_i^*) - \hat\epsilon_{\valpha_i}(h_i^*) \right| + \phi_{|\cH|}(\valpha_i, \vbeta, m, \delta) + \sum_{j \neq i} \alpha_{ij} D(\cD_i, \cD_j) \\
        &\leq \epsilon_{\valpha_i}(h_i^*) + 2\phi_{|\cH|}(\valpha_i, \vbeta, m, \delta) + \sum_{j \neq i} \alpha_{ij} D(\cD_i, \cD_j) \tag{Definition \ref{def:gen2}, hold with probability $1 - \delta$}\\
        &\leq \epsilon_i(h_i^*) + \left| \epsilon_i(h_i^*) - \epsilon_{\valpha_i}(h_i^*) \right| + 2\phi_{|\cH|}(\valpha_i, \vbeta, m, \delta) + \sum_{j \neq i} \alpha_{ij} D(\cD_i, \cD_j)\\
        &\leq  \epsilon_i(h_i^*) + 2\phi_{|\cH|}(\valpha_i, \vbeta, m, \delta) + 2 \sum_{j \neq i} \alpha_{ij} D(\cD_i, \cD_j) \tag{Lemma \ref{lemma:err_decomp}} 
    \end{align*}
    Notice that we use the generalization error bound twice, so the bound holds with probability at least $1 - 2\delta$ instead of $1 - \delta$. 
\end{proof}

\subsection{Proof of Corollaries \ref{crl:special} and \ref{crl:general}}

\subsubsection{Proof of Corollary \ref{crl:special}(1) }

\begin{maincorollary}[1] \label{crl:special:1}
When using VC-dimension bound (\ref{def:vc2}) as the quantity aware function, if $D(\cD_i, \cD_j) = 0, \forall i, j$, GFL minimizes the error bound with $\alpha_{ij} = \beta_j, \forall {j}$. 
\end{maincorollary}
\begin{proof}
    In the first case, $D(\cD_i, \cD_j)=0$. Let $Q=\sqrt{\frac{2d \log (2m + 2) + \log(4 / \delta)}{m}} $. Then
    \begin{align*}
        f(\valpha_i) &= 2 \phi_{|\cH|}(\valpha_i, \vbeta, m, \delta) = 2Q \sqrt{ \sum_{j=1}^N \frac{\alpha_{ij}^2}{\beta_j} } \\
        &= 2Q \sqrt{ \sum_{j=1}^N \frac{(\alpha_{ij} - \beta_j)^2}{\beta_j} + 1 } \\
        &= 2Q \sqrt{ \chi^2 \left( \valpha_i || \vbeta \right) + 1 }
    \end{align*}
    which achieves the minimum at $\valpha_i = \vbeta$.
\end{proof}

\subsubsection{Proof of Corollary \ref{crl:general} and \ref{crl:special}(2)}

\begin{maincorollary} \label{crl:general2}
When using VC-dimension bound (\ref{def:vc2}) as the quantity aware function, for a client $i$ with $m_i$ samples, if its coalition $\cC$ minimizes the error bound of Theorem \ref{thm:error}, then $\cC$ does not include any clients with distribution distance $D(\cD_i, \cD_j) > D_\text{thr}$, where $D_\text{thr} = \frac{\sqrt{2d \log (2m + 2) + \log (4  / \delta)}\sqrt{m}}{2 m_i}$.
\end{maincorollary}

\begin{proof}
    For the coalition $\cC$, if there is at least one client $j \in \cC$ with $\cD(\cD_i, \cD_j) > D_{\text{thr}} > 0$, we show that there exists a different coalition $\cC^- = \{k \in \cC: \cD(\cD_i, \cD_k) \leq D_{\text{thr}}\}$ with strictly smaller error bound. 

    We denote $\cC^+ = \{k \in \cC, k \neq i: \cD(\cD_i, \cD_k) > D_{\text{thr}}\}$ as the clients in the coalition with distribution differences strictly larger than $D_{\text{thr}}$, and $\cC^- = \{k \in \cC, k \neq i: \cD(\cD_i, \cD_k) \leq D_{\text{thr}}\} \cup \{i\}$ as the clients in the coalition with distribution differences smaller than or equal to $D_{\text{thr}}$ (including client $i$ itself). Notice that 
    \begin{itemize}[itemsep=0pt,topsep=0pt]
        \item $\cC = \cC^+ \cup \cC^-$, 
        \item $\cC^- \subsetneqq \cC, \cC^+ \neq \emptyset$, and 
        \item $i \in \cC^-$. 
    \end{itemize}
    Therefore, $\cC^-$ is a different coalition for client $i$. 
    Next, we prove that coalition $\cC^-$ has a strictly smaller error bound than $\cC$. For clarity, we denote $m_\cC = \sum_{j \in \cC} m_j$ and $\mu = \sqrt{2d \log (2m + 2) + \log (4  / \delta)}$. We first quantify the error bound for $\cC$. 
    \begin{align*}
        \text{error\_bound} (\cC) &= \epsilon_i(h_i^*) + 2 \mu \sqrt{\frac{1}{m} \sum_{j\neq i} \frac{\alpha_{ij}^2}{\beta_j} } + 2 \sum_{j\neq i} \alpha_{ij} D(\cD_i, \cD_j) \\
        &= \epsilon_i(h_i^*) + 2 \mu \sqrt{\frac{1}{m} \sum_{j \in \cC} \frac{ \left( \frac{\beta_j}{\sum_{k \in \cC} \beta_k}\right)^2 }{\beta_j}} + 2 \sum_{j \in \cC - \{i\}}  \frac{\beta_j}{\sum_{k \in \cC} \beta_k} D(\cD_i, \cD_j) \\
        &= \epsilon_i(h_i^*) + 2 \mu \sqrt{\frac{1}{m} \sum_{j \in \cC} \frac{\beta_j}{\left( \sum_{k \in \cC} \beta_k \right)^2}} + 2 \sum_{j \in \cC - \{i\}}  \frac{\beta_j}{\sum_{k \in \cC - \{i\}} \beta_k} D(\cD_i, \cD_j) \\
        &= \epsilon_i(h_i^*) + 2 \mu \sqrt{\frac{1}{\sum_{j \in \cC} m_j}} + 2 \sum_{j \in \cC - \{i\}}  \frac{m_j}{\sum_{k \in \cC} m_k} D(\cD_i, \cD_j) \\
        &= \epsilon_i(h_i^*) + 2 \mu \sqrt{\frac{1}{m_\cC}} + 2 \sum_{j \in \cC - \{i\}}  \frac{m_j}{m_\cC} D(\cD_i, \cD_j) 
    \end{align*}
    We can quantify the error bound for $\cC^-$ through the same steps. Then we can compare two error bounds. 
    \begin{align*}
        &\quad\ \ \text{error\_bound} (\cC) - \text{error\_bound} (\cC^-) \\
        &= \left(2 \mu \sqrt{\frac{1}{m_\cC}} + 2 \sum_{j \in \cC - \{i\}}  \frac{m_j}{m_\cC} D(\cD_i, \cD_j)\right) - \left(2 \mu \sqrt{\frac{1}{m_{\cC^-}}} + 2 \sum_{j \in \cC^- - \{i\}}  \frac{m_j}{m_{\cC^-}} D(\cD_i, \cD_j)\right) \\
        &= - 2 \mu \left( \sqrt{\frac{1}{m_{\cC^-}}} - \sqrt{\frac{1}{m_\cC}} \right) + 2 \left( \sum_{j \in \cC - \{i\}}  \frac{m_j}{m_\cC} D(\cD_i, \cD_j) - \sum_{j \in \cC^- - \{i\}}  \frac{m_j}{m_{\cC^-}} D(\cD_i, \cD_j)  \right)
    \end{align*}
    In the first term, 
    \begin{align*}
        \sqrt{\frac{1}{m_{\cC^-}}} - \sqrt{\frac{1}{m_\cC}} &= \frac{\frac{1}{m_{\cC^-}} - \frac{1}{m_\cC}}{\sqrt{\frac{1}{m_{\cC^-}}} + \sqrt{\frac{1}{m_\cC}}} \\
        &< \frac{\frac{1}{m_{\cC^-}} - \frac{1}{m_\cC}}{\sqrt{\frac{1}{m}} + \sqrt{\frac{1}{m}}} \\
        &= \frac{\sqrt{m}}{2} \left( \frac{1}{m_{\cC^-}} - \frac{1}{m_\cC} \right)
    \end{align*}

    In the second term, 
    \begin{align*}
        & \quad \sum_{j \in \cC - \{i\}}  \frac{m_j}{m_\cC} D(\cD_i, \cD_j) - \sum_{j \in \cC^- - \{i\}}  \frac{m_j}{m_{\cC^-}} D(\cD_i, \cD_j) \\
        &= \sum_{j \in \cC^- - \{i\}} \left(\frac{m_j}{m_\cC} - \frac{m_j}{m_{\cC^-}}\right) D(\cD_i, \cD_j) + \sum_{j \in \cC^+} \frac{m_j}{m_\cC} D(\cD_i, \cD_j) \\
        &> \sum_{j \in \cC^- - \{i\}} \left(\frac{m_j}{m_\cC} - \frac{m_j}{m_{\cC^-}}\right) D_{\text{thr}} + \sum_{j \in \cC^+} \frac{m_j}{m_\cC}D_{\text{thr}} \\
        &= \left(\sum_{j \in \cC - \{i\}}\frac{m_j}{m_\cC} - \sum_{j \in \cC^- - \{i\}}\frac{m_j}{m_\cC^-} \right) D_{\text{thr}} \\
        &= \left( \frac{1}{m_\cC^-} - \frac{1}{m_\cC} \right) m_i D_{\text{thr}}
    \end{align*}

    Put them together, we have
    \begin{align*}
        \text{error\_bound} (\cC) - \text{error\_bound} (\cC^-) &> -2 \mu \frac{\sqrt{m}}{2} \left( \frac{1}{m_{\cC^-}} - \frac{1}{m_\cC} \right) + 2 \left( \frac{1}{m_\cC^-} - \frac{1}{m_\cC} \right) m_i D_{\text{thr}} \\
        &= \left( \frac{1}{m_{\cC^-}} - \frac{1}{m_\cC} \right) \left( 2 m_i D_{\text{thr}} - \mu \sqrt{m} \right) \\
        &= 0
    \end{align*}
    Therefore, the coalition $\cC^-$ has a strictly smaller error bound than $\cC$. 
\end{proof}

\setcounter{maintheorem}{3}

\begin{maincorollary}[2] \label{crl:special:2}
When using VC-dimension bound (\ref{def:vc2}) as the quantity aware function, if $\min_{j \neq i} D(\cD_i, \cD_j) > \frac{\sqrt{2d \log (2m + 2) + \log (4  / \delta)}\sqrt{m}}{2 m_i}$, where $d$ is the VC-dimension of the hypothesis space, local training minimizes the error bound with $\alpha_{ii} = 1$ and $\alpha_{ij} = 0, \forall j \neq i$. 
\end{maincorollary}
\begin{proof}
    It is a special case for Corollary \ref{crl:general2}. Since $\forall i\neq j, D(\cD_i, \cD_j) > \frac{\sqrt{2d \log (2m + 2) + \log (4  / \delta)}\sqrt{m}}{2 m_i}$, each client's coalition should only include itself, which results in local training. 
\end{proof}

\subsection{Derivation of Client Discriminator} \label{eq:derive}

\begin{align*}
    D(\cD_i, \cD_j) &= \max_{h \in \cH} | \epsilon_i(h) - \epsilon(h) | \\
    &=\max_{h \in \cH} \left| \bbE_{(\vx, \vy) \in \cD_i} \ell(h(\vx), \vy) - \bbE_{(\vx, \vy) \in \cD_j} \ell(h(\vx), \vy) \right| \\
    &= \max_{f \in \cF} \left| \bbE_{(\vx, \vy) \in \cD_i} f(\vx, \vy) - \bbE_{(\vx, \vy) \in \cD_j} f(\vx, \vy) \right| \\
    &= \max_{f \in \cF} \left| \Pr_{(\vx, \vy) \in \cD_i} [f(\vx, \vy) = 1] - \Pr_{(\vx, \vy) \in \cD_j} [f(\vx, \vy) = 1] \right| \\
    &= \max_{f \in \cF} \left| \Pr_{(\vx, \vy) \in \cD_i} [f(\vx, \vy) = 1] +  \Pr_{(\vx, \vy) \in \cD_j} [f(\vx, \vy) = 0] - 1 \right| \\
    &= \max_{f \in \cF} \left| 2 \cdot \text{BalAcc} (f, \{\cD_i, 1\} \cup \{\cD_j, 0\}) - 1\right| 
\end{align*}
where 
\begin{align*}
    \text{BalAcc} (f, \{\cD_i, 1\} \cup \{\cD_j, 0\}) = \frac{1}{2} \left( \Pr_{(\vx, \vy) \in \cD_i} [f(\vx, \vy) = 1] +  \Pr_{(\vx, \vy) \in \cD_j} [f(\vx, \vy) = 0]  \right)
\end{align*}
is the balanced accuracy. 

\newpage
\section{Additional Experiments}

\subsection{Setup}

Here we provide more information on our experimental settings. Table \ref{tab:quantity} show the statistics of training/testing samples in three scenarios we consider. Notice that the testing data is NOT used during collaboration structure optimization or FL model training.
\begin{table*}[h]
\setlength{\tabcolsep}{0.9mm}{
\caption{Number of training / testing samples on each client \label{tab:quantity}}
\vskip 2mm
\begin{center}
\begin{small}
\begin{tabular}{l|c|c|c}
\toprule
Client & Label Shift (FashionMNIST) & Feature Shift (CIFAR-10) & Concept Shift (CIFAR-100) \\
\midrule
``Large'' (0-9) & 2,100 / 350 & 2,500 / 500  & 2,500 / 500 \\
``Small'' (10-19) & 14 / 350 & 340 / 500 & 120 / 500 \\
\bottomrule
\end{tabular}
\end{small}
\end{center}
}
\vskip -0.1in
\end{table*}

We run $C = 6, 8, 10$ on all three settings, and report the best result. Finally, we choose $C = 10$ for FashionMNIST and CIFAR-10 experiments, and $C = 8$ for CIFAR-100 experiments. 

Our code is available at \url{https://github.com/baowenxuan/FedCollab}.


\subsection{New Training Clients (RQ5)} \label{subsec:newclient}
\label{sec:exp_new}

In this part, we study whether new training clients can contribute to the FL system with a collaboration structure solved by {\algname}. We initialize a FedAvg system with 19 clients, leaving client 0 out. Client 0 operates local training in the first 200 rounds, and joins the FL system after 200 rounds (when FL models nearly converge). We use {\algname} to decide which coalition it should join. As shown in Figure \ref{fig:newc}, client 0 (``new'') receives a FL model with higher accuracy than the local model. Meanwhile, clients in the updated coalition (``clustered'') also benefit from the new training client since the FL model has additional performance gain after the new client 0 joins the training. 

\begin{figure*}[h]
\begin{center}
\centerline{\includegraphics[width=0.45\columnwidth]{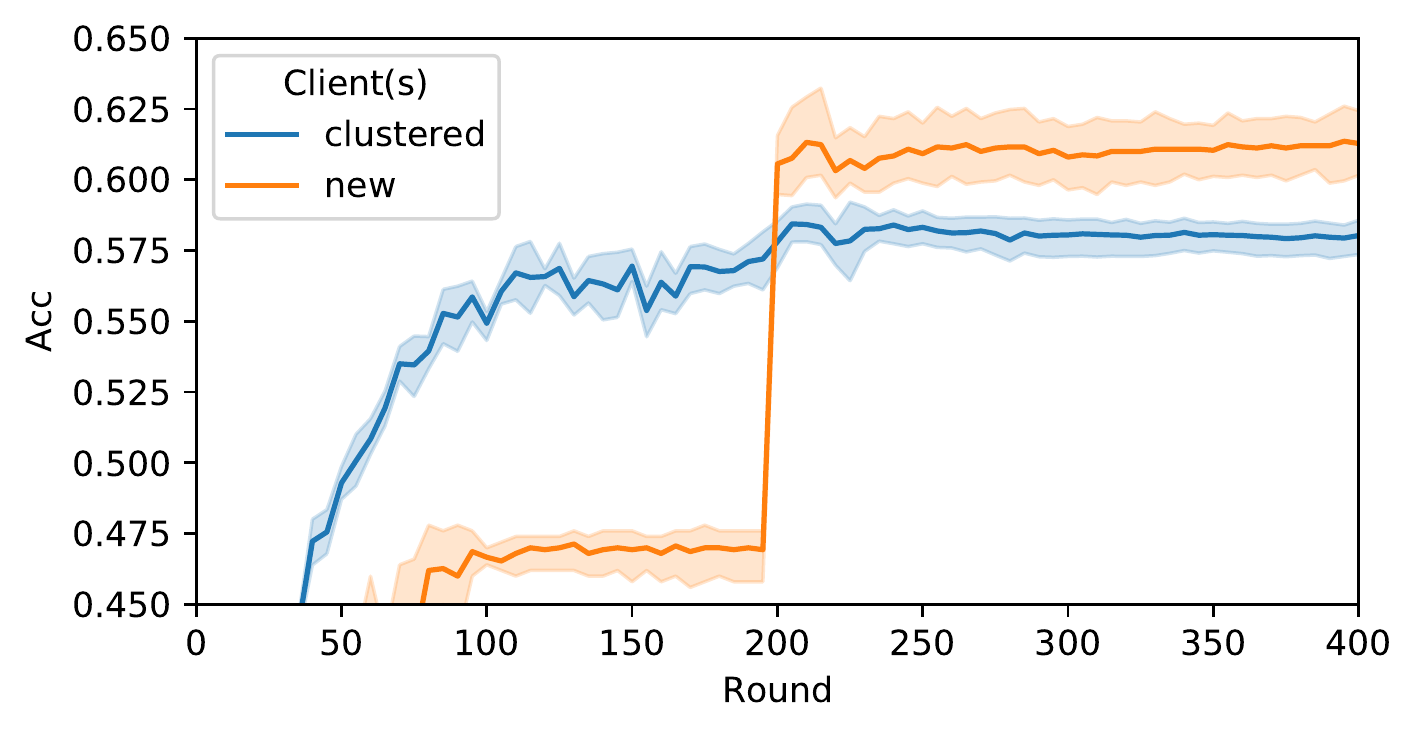}}
\caption{Utilizing new training clients on CIFAR-10 with feature shift}
\label{fig:newc}
\end{center}
\end{figure*}


\subsection{Convergence of {\algname} solver (RQ6)} \label{subsec:converge}

Algorithm \ref{alg:opt} is theoretically guaranteed to converge. In this part, we further study how many iterations it needs to converge, and whether it falls into local optima. Figure \ref{fig:convergence} visualizes the result of convergence. We re-run the {\algname} solver 100 times with different random seeds, and plot all trajectories indicating how the {\algname} loss changes w.r.t. inner iterations (line 3-6 in Algorithm \ref{alg:opt}). Notice that the {\algname} loss is evaluated for $N = 20$ times in each inner iteration. 

All the random runs converge within the first 60 inner iterations, while stopping within the first 80 inner iterations (since it requires one additional outer iteration to confirm convergence). Since the evaluation of {\algname} loss is very efficient, it only takes around 100ms to run Algorithm \ref{alg:opt} once. 

We also notice that a single run of Algorithm \ref{alg:opt} cannot guarantee the optimal solution. In many runs, {\algname} solver converges to a sub-optimal solution, which gives a collaboration structure of [[0..4], [5..9], [10..14], [15..19]]. Therefore, we use multiple random runs to refine the collaboration structure. 

\begin{figure}[h]
    \centering
    \subfigure[\# inner iteration v.s. loss value]{
    \includegraphics[width=0.45\columnwidth]{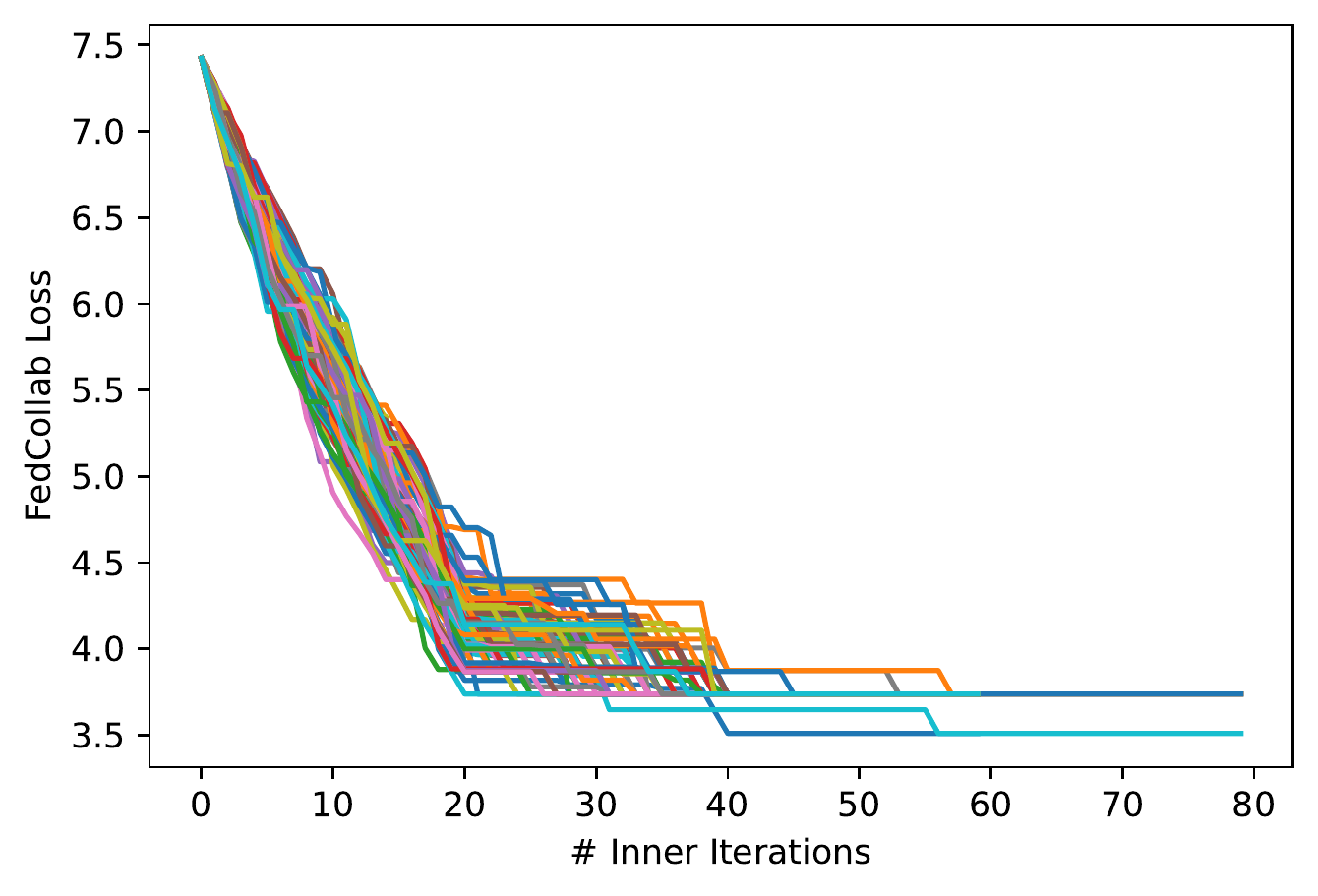}
    }
    \centering
    \subfigure[Histogram of \# inner iteration to converge]{
    \includegraphics[width=0.45\columnwidth]{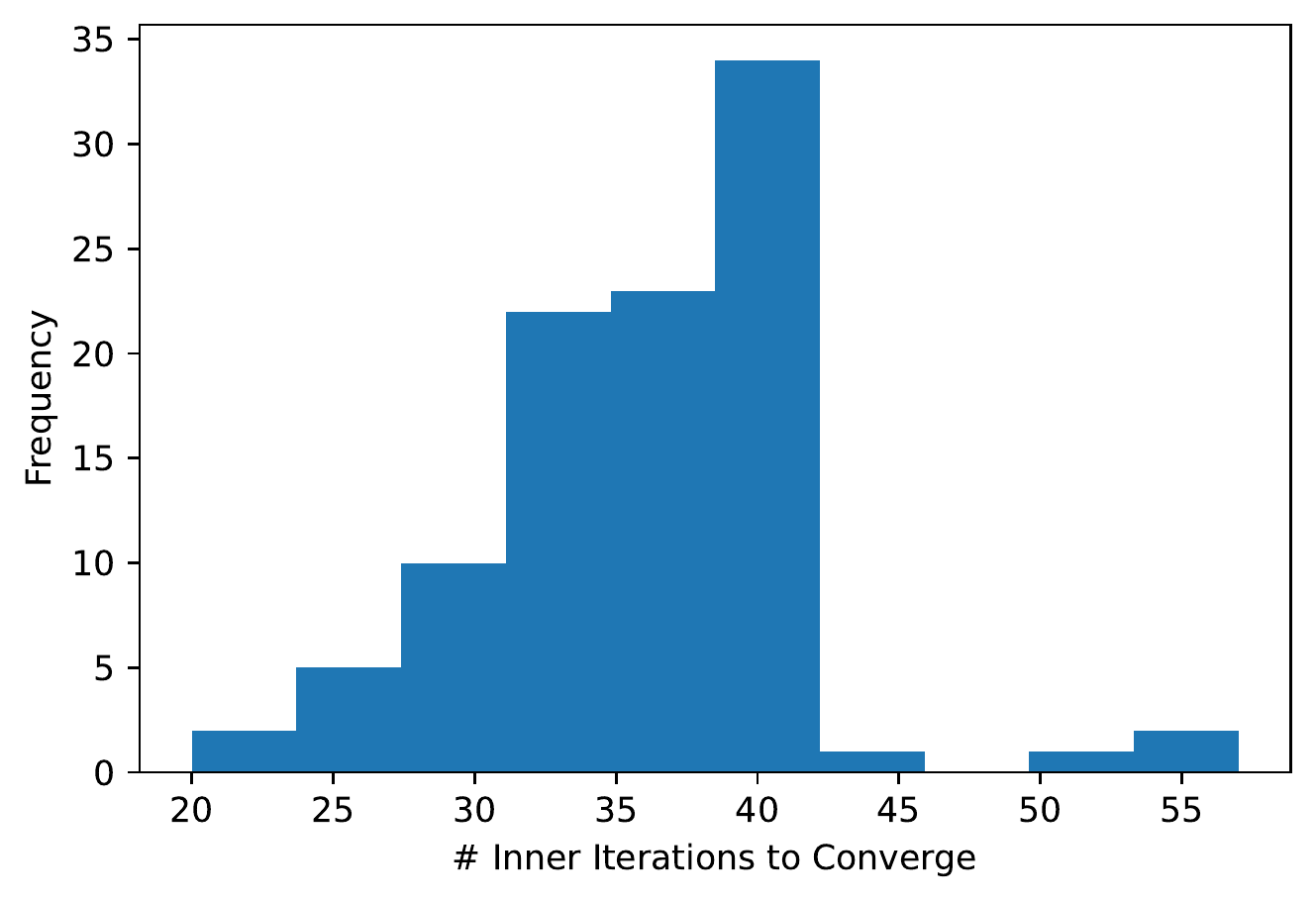}
    }
    \centering
    \caption{Convergence of Algorithm \ref{alg:opt} on CIFAR-10}
    \label{fig:convergence}
\end{figure}

\subsection{Computation and Communication Complexity}\label{subsec:complexity}

During the clustering step, {\algname} trains $\frac{N(N-1)}{2}$ client discriminators for $N$ clients, introducing $\cO(N^2)$ complexity. The complexity of {\algname} has the same order as many cross-silo FL algorithms, including MOCHA \cite{mocha}, FedFOMO \cite{fedfomo}, and PACFL \cite{pacfl}, which all model pairwise relationship among clients. Meanwhile, the training of client discriminators can be conducted in parallel: each client can train $N-1$ discriminators in parallel with other clients. 

\begin{table}[h!]
\caption{Comparison of computation and communication complexities (CIFAR-100 experiment) \label{tab:complexity}}
\begin{center}
\begin{small}
\vspace{2mm}
\begin{tabular}{lrr}
\toprule
Model & MACs & Params \\
\midrule
Client discriminator (MLP) & 104,600 & 105,001 \\
FL model (ResNet-18) & 37,220,352 & 11,181,642\\
\bottomrule
\end{tabular}
\end{small}
\end{center}
\end{table}

In the paper, to reduce the computation and communication constraints, we use a lightweight 2-layer MLP as the client discriminator, which is very efficient compared to the FL model (ResNet-18). We numerically evaluate their computation and communication complexities in the CIFAR-100 experiment. 
\begin{itemize}[itemsep=0pt,topsep=0pt]
    \item For computation cost, we count the number of multiply-add cumulations (MACs) for the forward pass of a single data point. 
    \item For communication cost, we count the number of parameters (Params). 
\end{itemize}
As shown in Table \ref{tab:complexity}, the MACs and Params for client discriminator are negligible compared to the FL model. Considering that each client needs to train $N-1 = 19$ client discriminators in total, our clustering step still only introduces $\sim 5.3 \%$ additional computation cost and $\sim 17.8 \%$ additional communication cost for each client.

\subsection{Discussion of Clustering During or before FL}\label{subsec:cluster_fl}

Compared to clustering during FL, clustering before FL has the following \textit{advantages}.

\begin{itemize}[itemsep=0pt,topsep=0pt]
    \item Clustering before FL is more stable and efficient. IFCA and FeSEM conduct clustering during FL. Their clustering results are influenced by the random initialization, and can easily converge to suboptima. To jump out of local optima, IFCA must conduct the whole FL training for multiple times, which is very inefficient. In comparison, {\algname} does not rely on any random initialization of the collaboration structure, and can refine the collaboration structure within only a few seconds. 
    \item Clustering before FL is more flexible. For clustering before FL, the clustering and FL phases are disentangled, which allows them to be seamlessly integrated with any GFL or PFL algorithms. Meanwhile, the convergence of clustering during FL algorithms may depend on specific FL algorithm. 
    \item Clustering before FL saves communication cost for outliers. For outlier clients that have significantly different distribution from all other clients, {\algname} allows them to form self-clusters, and they do not need to participate in the FL phase anymore (see blue cluster in Figure \ref{fig:overview}). It saves communication cost for these outlier clients and the server. It also prevents other clients from being negatively affected by outlier clients. 
\end{itemize}

For \textit{disadvantages}, clustering before FL requires each client's data set to be stable. In other words, the same client data set is used for clustering and FL. If the data sets for clustering and FL have different distributions or quantities, the optimal collaboration during clustering may not also be optimal for FL. However, this requirement is automatically satisfied for clustering during FL. 


\end{document}